\pdfoutput=1
\documentclass{article}

\usepackage{xcolor}

\usepackage[final]{corl_2026}

\usepackage{pgf}

\everymath=\expandafter{\the\everymath\displaystyle}
\makeatletter\@ifpackageloaded{underscore}{}{\usepackage[strings]{underscore}}\makeatother

\usepackage{amsmath}
\usepackage{amssymb}
\usepackage{mathtools}
\usepackage{amsthm}
\newtheorem{theorem}{Theorem}

\usepackage{algorithm}
\usepackage{algpseudocode}

\usepackage{graphicx}
\usepackage{subcaption}
\usepackage{wrapfig}
\usepackage{float}
\graphicspath{
  {figures/}
  {figures/1_teaser_real_world_criticality_plot_one_row/}
  {figures/2_hammer_sim_criticality_plots/}
  {figures/5_hardware_stackcube_demo_quality/}
  {figures/app_criticality_comparison_stackcube_sim/}
  {figures/app_criticality_comparison_peginsertionside_sim/}
  {figures/app_criticality_comparison_turnfaucet_sim/}
  {figures/app_hw_trajectory_quality_tapehang/}
}

\usepackage{booktabs}
\usepackage{multirow}
\usepackage{array}

\usepackage{enumitem}

\newcommand{\blfootnote}[1]{%
  \begingroup
    \renewcommand{\thefootnote}{}%
    \footnotetext{#1}%
  \endgroup
}

\title{Where Small Mistakes Matter: Quantifying State Criticality for Imitation Learning}

\author{
  Yusuf Umut Ciftci\textsuperscript{*}\\
  University of Southern California\\
  Stanford University\\
  \And
  Joonwon Kang\textsuperscript{*}\\
  Georgia Institute of Technology\\
  Stanford University\\
  \AND
  Ajay Suresha Sathya\\
  DIENS, Inria Paris, PSL University\\
  Stanford University\\
  \And
  Somil Bansal\\
  Stanford University\\
}

\begin{document}
\hypersetup{
  pdftitle={Where Small Mistakes Matter: Quantifying State Criticality for Imitation Learning},
  pdfauthor={Yusuf Umut Ciftci, Joonwon Kang, Ajay Suresha Sathya, Somil Bansal}
}
\maketitle
\blfootnote{\textsuperscript{*}Equal contribution.}

\begin{abstract}
Imitation learning typically treats every demonstration state as equally important, optimizing a behavior-cloning objective uniformly across the dataset. Yet not all states are created equal: in a small number of critical states, even a slight deviation from the expert action can lead to a failure, while in many other states such errors are recoverable and inconsequential.
We formalize state criticality in imitation learning as the worst-case degradation in downstream value induced by a bounded perturbation of the expert action at a given state.
We then propose \textsc{SCRIM}, a practical \textbf{S}tate \textbf{CRI}ticality \textbf{M}easure that estimates this quantity by training an adversarial disturbance policy and a value critic around a fixed imitation policy.
SCRIM assigns each demonstration state a score reflecting how vulnerable the task outcome is to small, policy-relevant action errors. We show that these scores can be used to reweight behavior-cloning losses, improving imitation at consequential states and increasing overall task performance. We further show that SCRIM enables data curation by selecting compact, high-value subsets of demonstrations that preserve performance while substantially reducing the amount of training data. We evaluate SCRIM across various simulated and real-world robotic manipulation tasks, and demonstrate that SCRIM improves policy training, enables efficient data selection, and provides interpretable insights into which moments in a demonstration are most critical for task success.
Project website: \url{https://sia-lab-git.github.io/SCRIM/}
\end{abstract}

\keywords{Data Quality, Imitation Learning, Robotic Manipulation}

\section{Introduction}
\label{sec:intro}

\begin{figure}[t]
  \centering
  \resizebox{\linewidth}{!}{\input{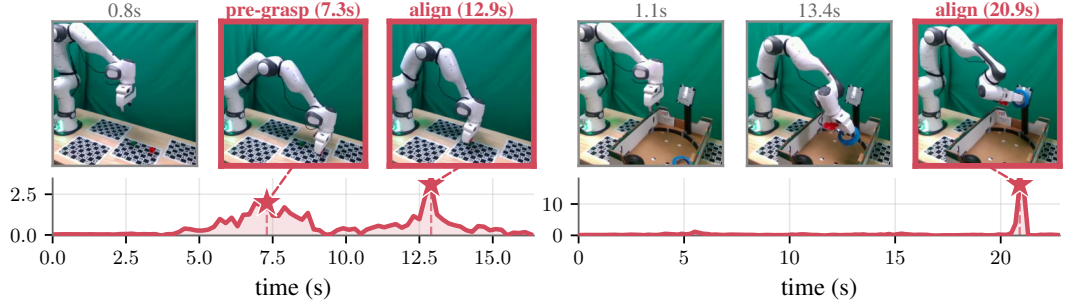}}
  \caption{
  Our state
  criticality measure, SCRIM, along a
  successful demonstration of cube stacking (left) and tape hanging (right),
  with a filmstrip of the rollout.
  SCRIM peaks at semantically critical moments where small action errors are likely to derail the task, such as aligning with the cube before grasp, aligning the cube for placement, and aligning the tape with the hook.
  }
  \vspace{-5pt}
  \label{fig:teaser}
\end{figure}

A simple and widespread Imitation learning (IL) approach, behavior cloning (BC), trains a policy to match the expert action at each demonstrated state. Recent advances in compute, large-scale robot datasets~\cite{ox2024open, Khazatsky-RSS-24}, visuomotor policy architectures~\cite{Chi-RSS-23}, and foundation models~\cite{zitkovich23a,Zhao-RSS-23,intelligence2025pi_} have substantially expanded the capabilities of BC policies, enabling robots to perform increasingly complex manipulation tasks. IL is especially attractive for manipulation, where contact-rich, nonsmooth dynamics~\cite{brogliato1996nonsmooth} and sparse task feedback make reinforcement learning (RL) difficult, even when an appropriate reward function is available.

Despite advances, BC remain brittle when it encounters states not well-covered by the demonstrations~\cite{imitationdataquality}. Improving robustness typically requires collecting more demonstrations, which is expensive. Moreover, useful recovery data often requires expert demonstrators to \textit{manually} anticipate critical moments, and demonstrate how to recover from them. Residual RL offers an alternative by learning a corrective policy on top of a BC policy~\citep{johannink2019residual, ankile2025residual, yuan2024policy}. However, residual RL remains challenging in sparse-reward manipulation settings due to high sample complexity, often requiring prohibitively many environment interactions or a carefully engineered simulation environment.

This paper investigates a complementary approach to improving BC policies: identifying \emph{where} accurate imitation matters most. Standard BC objectives treat all demonstrated states equally, yet manipulation tasks contain a small number of critical states where a small deviation from the expert action can derail the episode, while most others are forgiving. For example, in cube-stacking, the alignment phases immediately before grasping or placing the cube are critical even though they appear uneventful: a small action error can prevent a stable grasp or lead to an unstable placement. Errors of similar magnitude during earlier reaching motions, by contrast, are often recoverable in subsequent steps. Robust imitation therefore requires a state criticality measure that quantifies which states are most consequential for task success.
A reliable criticality measure could support targeted policy improvement. During training, it can reweight the BC objective to prioritize states where errors are most likely to cause failure. During data curation, it can identify high-criticality demonstrations, enabling smaller subsets to match or exceed full-dataset performance while reducing training and storage cost~\citep{Ding2025SurprisePA}.

Systematically quantifying state criticality for IL is nontrivial~\citep{Huang2018EstablishingAT}. A useful measure should not flag merely states with large actions, contact forces, or proximity to objects, but states where \emph{small, policy-relevant action errors} substantially degrade task performance. We formalize this as the worst-case increase in task failure under a bounded adversarial perturbation of the intended action, and propose \textsc{SCRIM}, a State Criticality Measure that learns it by training an adversarial residual policy and critic around a fixed BC policy. SCRIM produces per-state scores identifying where the policy is most vulnerable to small action errors.
We evaluate SCRIM on six manipulation tasks. Qualitatively, it identifies semantically meaningful critical phases---grasp alignment, contact timing---that geometric and value-based baselines miss (Fig.~\ref{fig:teaser}). Quantitatively, SCRIM-guided data selection and loss reweighting significantly improve BC over standard training and alternative criticality measures.

Concretely, this paper's contributions are as follows:
\textbf{(1)} introducing a formal notion of state criticality for imitation learning;
\textbf{(2)} proposing \textsc{SCRIM}, a practical algorithm for estimating state criticality;
\textbf{(3)} showing that SCRIM identifies semantically meaningful critical states in real-robot and simulated manipulation demonstrations;
\textbf{(4)} demonstrating that SCRIM can improve imitation learning through criticality-aware loss reweighting and data curation, leading to stronger BC policies than standard training and baseline criticality measures.

\section{Related Work}
\label{sec:related}

\noindent \textbf{Imitation learning and demonstration data.}
Data is the key ingredient in imitation learning. A policy can only be as good as the demonstrations it is trained on and it's ability to stick to the demonstrations. Both the quality and the composition of a dataset strongly shape downstream performance~\citep{imitationdataquality}. This has motivated a large body of work on \emph{data curation}, which selects or reweights existing demonstrations based on task distribution~\citep{remixdata, Parekh2025TowardsBB}, surrogate models ~\citep{Agia2025CUPIDCD, Dass2025DataMILSD}, quality labels or estimates~\citep{pmlr-v229-kuhar23a, Hejna2025RobotDC, Chen2025CuratingDU}, or by shaping the data \emph{at collection time}~\citep{Ciftci2024SAFEGILSG, Qiu2025SafetyAwareIL, Brandfonbrener2022VisualBT}. While others \cite{Sindhwani2018LearningCV, Mehta2024StableBCCC, Brandfonbrener2022VisualBT} improved robustness to deviations from the demonstrations. Underlying these efforts are the important of small deviations from expert actions and a recurring observation that not all states are equally important. \citet{Huang2018EstablishingAT} argue that the essence of a policy is captured by a few \emph{critical states} where taking the right action is essential. We, formalize the importance of each state with a criticality score and use this single measure both to curate demonstrations and to reweight the imitation objective toward the states where small mistakes matter most.

\noindent \textbf{Residual RL.}
Residual RL trains a \emph{corrective} policy on top of a fixed base
controller and has become a popular recipe for refining
imitation policies~\citep{johannink2019residual, yuan2024policy, ankile2025residual, vla_res_rl, Ankile2024FromIT} for improved task success. We, however, aim to create a tool for increased understanding of imitation learning alongside performance improvement.
Moreover, although more sample efficient than full fine tuning of the base policy learning, such corrections still demand many environment
interactions, especially under the sparse rewards
typical of manipulation. We adopt the same residual
formulation but invert its objective. Rather than learning to correct the base
policy, we train an adversarial residual that \emph{degrades} it. Because actions that break a
task are more abundant than the corrective actions, this
adversarial residual is an easier exploration problem---an advantage
confirm empirically (Section~\ref{sec:experiments}).

\section{Background and Problem Formulation}
\label{sec:background}
We consider discrete-time sequential decision-making problems modeled as a Markov decision process (MDP), $\mathcal{M} = (\mathcal{S}, \mathcal{A}, P, r, \rho_0, \gamma)$,
where $\mathcal{S}$ is the state space, $\mathcal{A}$ is the action space,
$P(s' \mid s, a)$ is the transition kernel,
$r : \mathcal{S} \times \mathcal{A} \to \mathbb{R}$ is the reward function,
$\rho_0$ is the distribution over initial states, and
$\gamma \in [0, 1)$ is the discount factor.
A (stationary) policy $\pi(a \mid s)$ maps states to a distribution over actions; when the policy is deterministic, we write $a = \pi(s)$.

\noindent \textbf{Imitation learning.}
Given a demonstration dataset
$\mathcal{D} = \{ (s_i, a_i) \}_{i=1}^{N}$ collected by an expert policy
$\pi_E$, where $a_i \sim \pi_E(\cdot \mid s_i)$,
behavior cloning (BC) casts imitation learning as supervised learning, fitting a
parameterized policy $\pi_\theta$ to reproduce expert actions by
minimizing the loss,
\begin{equation}
  \mathcal{L}_{\mathrm{BC}}(\theta) =
    \mathbb{E}_{(s, a) \sim \mathcal{D}}
    \big[ \ell\big( \pi_\theta(s), a \big) \big],
  \label{eq:bc-objective}
\end{equation}
where $\ell$ is an action-prediction loss, e.g.\ the squared error
$\| \pi_\theta(s) - a \|_2^2$.
This objective weights all demonstration states equally. However, action errors are not equally consequential across states: in some states, a small deviation from the expert action may be easily recoverable, while in others it may lead to task failure.

\noindent \textbf{Reinforcement learning.}
Our analysis borrows the value-function machinery of reinforcement learning (RL), though imitation itself uses no reward signal. RL seeks a policy maximizing the expected discounted return $J(\pi) = \mathbb{E}_{\tau \sim \pi} \left[ \sum_{t=0}^{\infty} \gamma^t \, r(s_t, a_t) \right]$ over trajectories $\tau = (s_0, a_0, s_1, \dots)$ with $s_0 \sim \rho_0$, $a_t \sim \pi(\cdot \mid s_t)$, and $s_{t+1} \sim P(\cdot \mid s_t, a_t)$. For a fixed $\pi$, the \emph{Q-function} conditions this return on the first state and action,
\begin{equation}
  Q^{\pi}(s, a) = \mathbb{E}_{\tau \sim \pi}
    \left[ \sum_{t=0}^{\infty} \gamma^t \, r(s_t, a_t)
      \,\middle|\, s_0 = s, \, a_0 = a \right],
  \label{eq:q-function}
\end{equation}
and satisfies the Bellman equation $Q^{\pi}(s, a) = r(s, a) + \gamma \,
  \mathbb{E}_{s' \sim P(\cdot \mid s, a),\, a' \sim \pi(\cdot \mid s')}
  \!\big[ Q^{\pi}(s', a') \big]$. We will use $Q^\pi$ to assess how sensitive task performance is to the action at a given state.

\noindent \textbf{Problem statement.}
\label{sec:problem}
Our goal is to quantify the \emph{criticality} of a demonstration state for an imitation policy $\pi$. Intuitively, a state is critical if the success of the task is highly sensitive to the action selected at that state. Equivalently, critical states are those in which small action errors can substantially degrade downstream performance or cause task failure. In contrast, non-critical states admit a larger set of actions from which the policy can still recover and complete the task.

Given a learned imitation policy $\pi$ and access to rollouts of this policy, we seek to learn a state-dependent criticality measure that identifies where precise action selection matters most. Such a measure can be used to diagnose weaknesses of an imitation policy, prioritize data collection or policy improvement, and trigger targeted intervention during deployment.

\section{Method} \label{sec:method}
Our goal is to identify states in which small action errors can have large downstream consequences. To do so, we treat a learned imitation policy $\pi$ as a fixed \emph{base} controller and ask how much an adversarial, bounded perturbation to its action can degrade task outcome. This leads to a residual MDP in which the only decision variable is the perturbation applied around the base action.

\noindent \textbf{Adversarial residual MDP.}
At each state, we decompose the executed action into the action proposed by the base policy and a bounded residual disturbance:
\begin{equation}
  a = \pi(s) + \delta,
  \qquad
  \delta \in \Delta_\epsilon \coloneqq
    \{\, \delta \in \mathcal{A} : |\delta_j| \le \epsilon_j
      \;\; \forall j \,\},
  \label{eq:residual-action}
\end{equation}
where $\epsilon \in \mathbb{R}^{d}_{\ge 0}$ defines a per-coordinate perturbation budget and $d=\dim(\mathcal{A})$. This budget models the magnitude of the small action errors that an imitation policy might make at a state.
Freezing $\pi$ and treating the residual $\delta$ as the only decision
variable induces a \emph{residual MDP}
$\widetilde{\mathcal{M}}^{\,\epsilon}_{\pi}
  = (\mathcal{S}, \Delta_\epsilon, \widetilde{P}_\pi, \widetilde{r}_\pi,
     \rho_0, \gamma)$,
whose transition kernel and reward are those of $\mathcal{M}$ evaluated at the
disturbed action,
\begin{equation}
  \widetilde{P}_\pi(s' \mid s, \delta) = P\big(s' \mid s, \pi(s) + \delta\big),
  \qquad
  \widetilde{r}_\pi(s, \delta) = r\big(s, \pi(s) + \delta\big).
  \label{eq:residual-mdp}
\end{equation}
The state space, initial-state distribution, and discount factor are inherited from the original MDP $\mathcal{M}$. A disturbance policy $\nu(\delta \mid s)$ acts only through the residual action space $\Delta_\epsilon$; setting $\nu \equiv 0$ recovers the unperturbed execution of the base policy $\pi$.

\noindent \textbf{Critic.}
We instantiate the residual MDP reward (Eq.~\ref{eq:residual-mdp}) with a sparse, outcome-based signal.
A rollout receives a reward $r = -1$ if it \emph{succeeds} and $r = 0$ otherwise.
Under the standard maximize-return objective $J(\pi)$, successful trajectories therefore have lower, more negative value, while failed or non-successful trajectories have value closer to zero. Consequently, larger values correspond to worse outcomes.

Next, we train an adversarial disturbance policy on this reward in the residual MDP and retain its converged critic $\bar{Q}(s,a)$.
$\bar{Q}$ is always queried on the combined action executed by the environment, $a = \pi(s) + \delta$. Therefore, an increase in $\bar{Q}(s,a)$ indicates that the disturbance has made the trajectory more failure-prone.

\noindent \textbf{\textcolor{methoddarkred}{S}tate \textcolor{methoddarkred}{CRI}ticality \textcolor{methoddarkred}{M}easure (\textcolor{methoddarkred}{SCRIM}).}
We are now ready to define criticality: a state $s$ is \emph{critical} if some in-budget disturbance $\delta \in \Delta_\epsilon$ applied at $s$ substantially degrades the task outcome relative to faithfully following $\pi$.
Conversely, $s$ is \emph{forgiving} if no admissible disturbance at $s$ changes the outcome appreciably, compared to unperturbed execution.
Thus, the criticality of a state is the worst-case degradation in outcome value induced by a small action disturbance.

SCRIM precisely measures this worst-case degradation in the outcome value. Since we evaluate criticality on demonstration states, we define the ideal score relative to the expert action $\pi_E(s)$:
\begin{equation}
  C_\epsilon(s)
    = \max_{\delta \in \Delta_\epsilon} \bar{Q}\big(s,\, \pi_E(s) + \delta\big)
      - \bar{Q}\big(s,\, \pi_E(s)\big)
    \;\ge\; 0 .
  \label{eq:criticality}
\end{equation}
The first term is the worst outcome value attainable by an admissible perturbation of the expert action, while the second term is the outcome value of executing the expert action itself.
Because larger $\bar{Q}$ corresponds to worse outcomes under our reward convention, $C_\epsilon(s)$ is large when a small admissible action error can make failure substantially more likely, and close to zero when the state is forgiving.

In practice, we estimate \eqref{eq:criticality} with an adversary $\delta_\pi(s)$ learned jointly with the critic. Both the critic $\bar{Q}$ and the adversary are trained in the residual MDP induced by the learned imitation policy $\pi$ as the base controller. At evaluation time, however, we compute the criticality score around the expert action $\pi_E(s)$ on demonstration states:
\begin{equation}
  \Delta Q(s)
    = \bar{Q}\big(s,\, \pi_E(s) + \delta_\pi(s)\big)
    - \bar{Q}\big(s,\, \pi_E(s)\big) .
  \label{eq:delta-q}
\end{equation}
This evaluation asks whether a small prediction error around the demonstrated action would substantially increase the likelihood of downstream failure.
A state is \emph{critical} when $\Delta Q(s)$ is large: a small disturbance drives the value toward $0$, flipping a likely success into a likely failure, so the task is sensitive to exactly how the action is executed there.
A state is \emph{forgiving} when $\Delta Q(s) \approx 0$, i.e.\ no admissible disturbance appreciably changes the downstream outcome.

\noindent \textbf{Learning the adversary.}
The adversary policy is a
network $\mu_\theta$ that maps a state to a disturbance,
$\delta_\theta(s) = \epsilon \odot \tanh\!\big(\mu_\theta(s)\big)$,
where the $\tanh$ bounds each coordinate in $[-1,1]$ and $\epsilon$ rescales the output to the per-coordinate budget. Thus, $\delta_\theta(s) \in \Delta_\epsilon$ for every state.
We train $\mu_\theta$ jointly with the critic using soft actor-critic (SAC) in the residual MDP. The disturbance actor maximizes
$\bar{Q}\big(s, \pi(s) + \delta_\theta(s)\big)$,
i.e.\ it seeks the most failure-prone admissible disturbance, and at convergence it provides the $\delta_\pi(s)$ used in Eq.~\eqref{eq:delta-q}. See Algorithm~\ref{alg:scrim} (Appendix~\ref{sec:appendix-algorithm}).

The critic is implemented as a clipped double-$Q$ pair and trained using temporal-difference targets with Polyak-averaged target networks. We denote the conservative critic used for scoring by $\bar{Q}(s,a) = \min_i Q_{\bar{\phi_i}}(s,a).$
The actor is entropy-regularized with an automatically tuned temperature. During training, transitions $(s,\delta,r,s')$ are collected by executing the disturbed action
$a = \operatorname{clip}_{\mathcal{A}}\big(\pi(s) + \delta_\pi(s)\big)$
in the environment and storing the resulting transition in a replay buffer. This ensures that the adversary is trained on the same clipped physical actions that critic evaluates.

\noindent \textbf{Sampling-distribution calibration}
Both downstream uses of SCRIM (Section~\ref{sec:experiments}) first turn the
per-state criticality scores into a sampling distribution over the
demonstration states through a single, measure-agnostic pipeline.
Specifically, we $z$-score standardize the criticality scores to $z_i$ and map them to the temperature-$T$ softmax
$p_i(T) \propto \exp(z_i / T)$.
We choose the temperature $T$ so that the resulting distribution reaches a common target sharpness, measured by its fractional effective sample size $\mathrm{ESS}\%$.
We defer the standardization procedure, the definition of $\mathrm{ESS}$, and the temperature-selection details to Appendix~\ref{sec:appendix-calibration}.
The calibrated criticality scores can then be used for policy improvement, e.g., by upweighing critical samples.

\section{Experiments}
\label{sec:experiments}
The experiments are designed to answer four questions: (i)~does SCRIM identify intuitive critical states in a manipulation task, and more
reliably than existing criticality heuristics?, (ii)~does weighting behavior cloning by SCRIM improve policy performance?,
(iii)~can SCRIM curate a small subset of demonstrations that retains, or
improves on, full-data performance? (iv)~can SCRIM be feasibly evaluated on hardware via real world RL, and does it improve a real-world VLA policy? We study (i)--(iii) in simulation (Section~\ref{sec:exp-sim}) and validate (iv) on hardware (Section~\ref{sec:exp-hw}).

\subsection{Simulation Results}
\label{sec:exp-sim}
We consider four diverse manipulation tasks spanning a range of difficulty levels and requiring \textit{precise} contact-rich interaction making them suitable for studying state criticality:
stacking cubes, faucet turning, nail hammering and peg insertion.
The tasks are spread across two benchmark environments: \emph{Hammer} is the MetaWorld~\cite{yu2020meta} task (grasp a hammer and drive in a nail), with $84{\times}84$ RGB and a $4$-dim proprioceptive vector as observations and a
$4$-dim end-effector action; its BC policy is a deterministic CNN$+$MLP trained
on $20$ expert demonstrations. \emph{StackCube}, \emph{PegInsertion} and \emph{TurnFaucet}
are ManiSkill2~\cite{gumaniskill2} tasks with state observations and a Behavior
Transformer~\cite{shafiullah2022behavior} BC policy with $20$ time-steps of history; trained on $974$, $986$ and $1008$ demonstrations across 10 faucet models, respectively.
All actions are normalized to
$[-1,1]^{d_a}$.

To estimate criticality we freeze $\pi$ and train the SAC adversary against it
(Section~\ref{sec:method}) with a sparse success reward and a small in-budget
disturbance bound ($0.20$ per axis for MetaWorld task, $0.005$ for the ManiSkill2
tasks).
To facilitate comparison, every criticality measure---ours and the baselines---is mapped to a
sampling distribution by the same pipeline, calibrated for every measure to the same $\mathrm{ESS}$, effective sample size, so that the distributions are similarly peaked, differing only in \emph{which} states they emphasize.
We report task success rate and defer per-environment details to Appendix ~\ref{sec:appendix-sim}.

\noindent \textbf{Baselines.}\label{sec:baselines} SCRIM is compared against the following five baselines, each of which assigns a scalar criticality score to state
$s$ usable for both data weighting and data selection similarly to SCRIM:
(i)~\emph{uniform sampling}, the standard i.i.d.\ imitation-learning baseline that treats every state as equally important; (ii)~\emph{distance to object of interest}, a geometric heuristic that scores states by the (negative) end-effector--object distance and requires a known target pose; (iii)~\emph{action magnitude}, which scores states by $-\|a\|$ on the premise that intricate, contact-rich manipulation uses smaller commanded actions than coarse free-space motion; (iv)~\emph{policy $Q$-function variance}~\citep{Karino2020Identifying}, the empirical variance $\operatorname{Var}_{a \sim \mathcal{U}(\mathcal{A})}\!\big[ Q^{\pi_b}(s, a) \big]$ of the base policy's Q-value over actions sampled uniformly from $\mathcal{A}$; and (v)~\emph{$Q$-value discrepancy}~\citep{Huang2018EstablishingAT}, the gap $\max_{a} Q^{\pi}(s, a) - \tfrac{1}{|\mathcal{A}|} \sum_{a} Q^{\pi}(s, a)$ between the best and average Q-value at a state. More detailed baseline descriptions are presented in the appendix in Sec.~\ref{sec:appendix-baselines}.

\begin{figure}[t]
  \centering
\resizebox{0.9\linewidth}{!}{\input{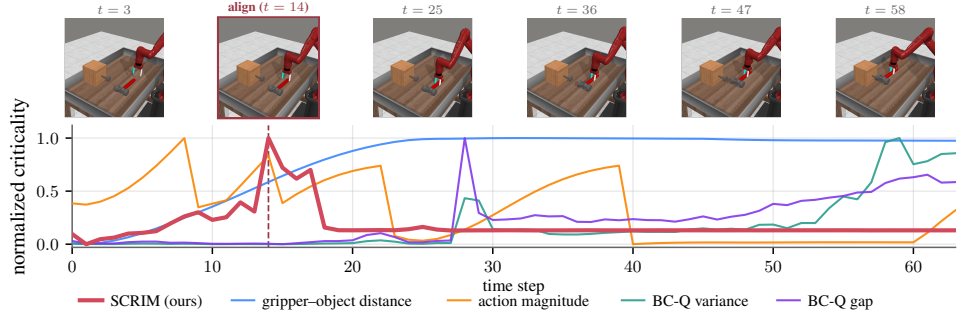}}
  \caption{\textbf{Per-method criticality along Hammer (MetaWorld) task demonstration.}
  All measures min--max normalized to a common scale.}
  \label{fig:criticality_comparison}
\end{figure}

\noindent \textbf{Criticality estimates.}
To qualitatively analyze the criticality measures, Fig.~\ref{fig:criticality_comparison} plots each score along a single hammer demonstration.
SCRIM rises as the gripper approaches the hammer immediately before the grasp,
identifying the alignment over the handle as the most critical phase.
This state is not obviously critical from task progress alone: the gripper has not yet contacted the object, but a small action error can prevent a stable grasp of the handle, causing the hammer to fall and derailing the episode.
In contrast, with good alignment, the grasp itself and the states far from the alignment phase are forgiving, with action errors recoverable in subsequent steps.

\begin{wrapfigure}{r}{0.25\textwidth}
  \centering
  \vspace{-19pt}
  \includegraphics[width=\linewidth]{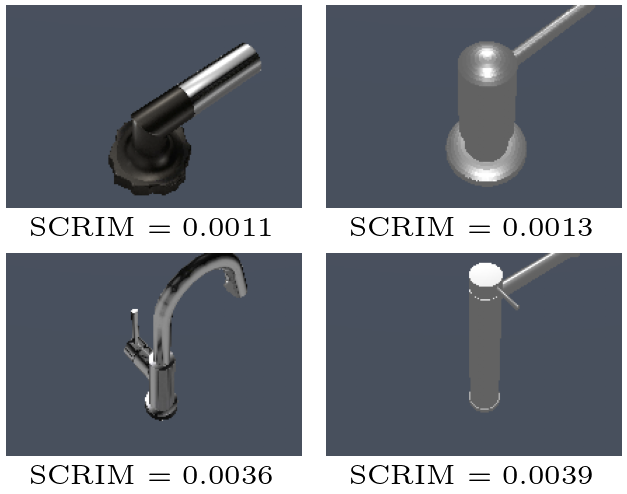}
  \caption{Faucet models ordered by mean SCRIM over the demonstrations, from the two
  easiest (top) to the two hardest (bottom).}
  \vspace{-14pt}
  \label{fig:faucet_difficulty}
\end{wrapfigure}

By contrast, the baselines peak away from this alignment phase. \emph{Distance to object} rises as the gripper nears the hammer but is flat afterward, failing to separate the post-grasp states; \emph{action magnitude} spikes at every phase transition, but there are several such transitions and not necessarily critical. The value-based baselines, \emph{Q-function variance} and \emph{Q-value discrepancy}, miss the alignment phase, peaking narrowly during contact-rich phases where the learned $Q$-function is most sensitive to action variation.
However, sensitivity across arbitrary actions is not necessarily the same as vulnerability to small, policy-relevant action errors. SCRIM directly targets the latter by measuring the worst-case degradation induced by bounded perturbations around the demonstrated action.
Additionally, the value-based baselines aggregate expectations or variances across \textit{all} actions, requiring an accurate $Q$-function everywhere.
In contrast, SCRIM only requires reliable value estimates in the local neighborhood of the policy actions explored by the adversary, focusing the estimate on the perturbations most relevant to imitation errors.

\noindent \textbf{SCRIM as a measure of task difficulty.}
Because SCRIM is computed from a single critic trained across demonstrations of the same task, its scores are comparable across trajectories and task variations and appears to be an effective proxy for the relative task difficulty.
We illustrate this on the \emph{TurnFaucet} task, where different faucet geometries induce different levels of manipulation difficulty.
Ordering faucet models by their SCRIM yields an intuitive easy-to-hard ranking (Fig.~\ref{fig:faucet_difficulty}):
models with broad, forgiving handles have low mean SCRIM, whereas models requiring precise grasping and turning have higher mean SCRIM.

\noindent \textbf{Weighted imitation.}
Demonstration datasets are dominated by non-critical states, with easily recoverable free-space motions heavily outnumbering the few critical states where a small action error can cause failure.
Uniform BC spreads its capacity uniformly across samples, decreasing the imitation accuracy in these decisive states, which can now be addressed using SCRIM.
Taking the calibrated distribution $p$ from Section~\ref{sec:method}
(Eq.~\ref{eq:softmax} in Appendix~\ref{sec:appendix-calibration}), we form
 mean-one per-sample weights $w_i = N\,p_i$ and train a fresh policy by minimizing the BC objective (Eq.~\ref{eq:bc-objective}) reweighted by $w_i$, equivalently sampling each minibatch from $p$ with replacement so the policy sees the decisive states more often during training.

Table~\ref{tab:results_weighted} reports the success rates for SCRIM-weighted BC and the various baselines. SCRIM weighting improves the average success rate over uniform BC, gaining on three of the four tasks and essentially matching it on PegInsertion. It also outperforms the other baseline weights: the strongest, Q-function variance, lands barely above uniform BC, while the geometric heuristics fare poorly because their weighting underweights genuinely critical states. This policy improvement provides strong evidence that SCRIM is a more effective criticality measure than the baselines.

\begin{figure}[t]
  \centering
  \resizebox{0.9\linewidth}{!}{\input{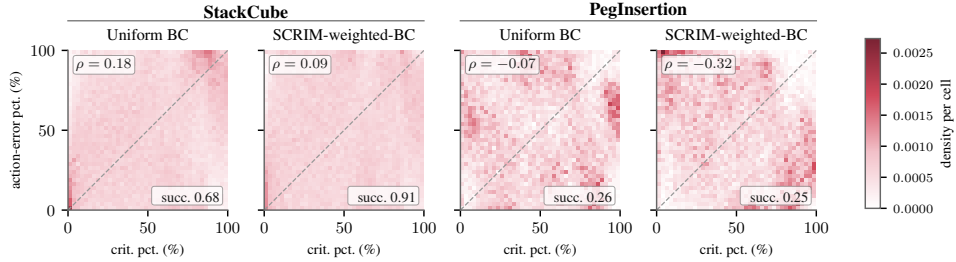}}
  \caption{\textbf{Action error vs.\ state criticality.} Histograms of BC action-error percentile against SCRIM- percentile on two tasks, each under Uniform BC and SCRIM-weighted BC.
  Each panel is annotated with its Spearman correlation $\rho$ and success
  rate.}
  \vspace{-7pt}
  \label{fig:error_vs_criticality}
\end{figure}

Figure~\ref{fig:error_vs_criticality} analyzes more closely the distribution of action-error percentile vs SCRIM percentile and its influence on policy improvement. For uniform BC on \emph{StackCube} demonstrations, the action error rank-correlates positively with criticality, with the distribution more heavily concentrated in regions with higher criticality and higher action-error. SCRIM-weighting flattens this relationship reducing error precisely at these states and raising the success rate. %
Conversely, SCRIM weighting appears to improve base policy success rates only when the base policy exhibits higher action errors at critical states, as seen for the \emph{PegInsertion} task. This reveals a limitation of policy improvement using SCRIM: due to lack of RL-style exploration for new behaviors, it cannot improve overall performance beyond demonstration data if the base policy is already accurate at critical states.

\begin{table}[t]
  \centering
  \caption{\textbf{Weighted training.} Success rates (\%) over 100 rollouts. Each criticality measure is
    used to weight the sampling of states during training (Section~\ref{sec:baselines}). %
    Best per task in \textbf{bold}.}
  \label{tab:results_weighted}
  \setlength{\tabcolsep}{3pt}
  \renewcommand{\arraystretch}{0.9}
  \small
  \begin{tabular}{lccccc}
    \toprule
    Method & StackCube & TurnFaucet  & PegInsertion & Hammer & Average \\
    \midrule
    Uniform sampling (BC)    & $68.0$ & $39.0$  & $\mathbf{26.0}$ & $61.0$ & $48.5$ \\
    Distance to object       & $4.0$  & $4.0$  & $0.0$ & $0.0$  & $2.0$  \\
    Action magnitude         & $41.0$ & $21.0$ & $5.0$ &$69.0$ &  $34.0$ \\
    Q-function variance      & $77.0$ & $44.6$ & $10.0$ & $65.0$ & $49.2$ \\
    Q-value discrepancy~\citep{Huang2018EstablishingAT} & $74.0$ & $33.0$ & $10.0$ & $67.0$ & $46.0$ \\
    \textbf{Ours}            & $\mathbf{91.0}$ & $\mathbf{47.0}$ & $25.0$ & $\mathbf{83.0}$ & $\mathbf{61.5}$ \\
    \bottomrule
  \end{tabular}
\end{table}

\begin{wraptable}{r}{0.48\linewidth}
  \centering
  \small
  \setlength{\tabcolsep}{3pt}
  \caption{\textbf{Matched interaction.} Success (\%) of residual RL~\cite{yuan2024policy}
    and Ours. Match: extra interaction residual RL needs to reach Ours.}
  \label{tab:residual_rl_matched}
  \begin{tabular}{lcccc}
    \toprule
    Task & Budget & Res.\ RL & Ours & Match \\
    \midrule
    StackCube & 100k & 0.0 & \textbf{85.0} & 7.5$\times$ \\
    TurnFaucet & 200k & 16.0 & \textbf{52.0} & 1.5$\times$ \\
    PegInsertion & 500k & 0.0 & \textbf{18.0} & 5.8$\times$ \\
    Hammer & 100k & \textbf{79.0} & 74.0 & --- \\
    \bottomrule
  \end{tabular}
\end{wraptable}

\noindent \textbf{Residual policy to improve performance.}
A natural alternative for policy improvement is to train a residual RL policy
that \emph{corrects} the base policy rather than SCRIM's adversarial RL approach. Our intuition is actions that break a task are more abundant than corrective actions that improve a policy, thereby making adversarial RL a fundamentally easier problem due to RL's sampling-based exploration. To validate this intuition, we also trained a residual-RL baseline~\citep{yuan2024policy} using SAC, where a residual actor learns a bounded corrective delta on the frozen base.
Under a \emph{matched} interaction budget, residual RL falls below the base policy it is meant to correct on all three ManiSkill tasks, while SCRIM-weighted BC already improves on it for StackCube and TurnFaucet; residual RL needs several times more interaction to catch up (Table~\ref{tab:residual_rl_matched}). Hammer is the exception: residual RL improves the base policy readily and leads early, although SCRIM reaches the same peak
success within 500k interactions. SCRIM's interaction efficiency is therefore task-dependent, and the two approaches are complementary.

\noindent \textbf{Curation of datasets that amplify critical signals }
We next investigate SCRIM for data curation: selecting an informative subset of demonstrations under a data budget.
Using the same softmax conversion of the criticality scores as in training, we sample without replacement to retain $25\%$ of the dataset and then train standard BC on this subset (Table~\ref{tab:results_selection}).  SCRIM attains the best average success rate and wins on three of the four tasks, with uniform selection narrowly ahead on PegInsertion. With only a quarter of the data, SCRIM-selected subsets match or exceed uniform BC trained on the \emph{full} dataset, indicating its effectiveness for task-relevant data curation.
\begin{table}[h]
  \centering
  \caption{\textbf{Data selection.} Success rates (\%) over 100 rollouts. Each criticality measure is converted into a sampling distribution over demonstration states, from which $25\%$ of the states are drawn without replacement; standard i.i.d.\ behavior cloning is then trained on this subset (Section~\ref{sec:baselines}). For reference, the top panel reports behavior cloning and \textbf{Ours} trained on the full (100\%) dataset. Best per column among the 25\% methods in
    \textbf{bold}.
    $^{\dagger}$\,Ours at 100\% is our criticality-weighted training of
    Table~\ref{tab:results_weighted}, included here as reference.}
  \label{tab:results_selection}
  \setlength{\tabcolsep}{3pt}
  \renewcommand{\arraystretch}{0.9}
  \small
  \begin{tabular}{llccccc}
    \toprule
    Data & Method & StackCube & TurnFaucet & PegInsertion & Hammer & Average \\
    \midrule
    \multirow{2}{*}{\rotatebox[origin=c]{90}{100\%}}
      & Uniform (BC)             & $68.0$ & $39.0$ & $26.0$ & $61.0$ & $48.5$ \\
      & Ours$^{\dagger}$          & $91.0$ & $47.0$ & $25.0$ & $83.0$ & $61.5$ \\
    \midrule
    \multirow{6}{*}{\rotatebox[origin=c]{90}{Subset (25\%)}}
      & Uniform sampling         & $73.0$ & $41.0$ & $\mathbf{14.0}$ & $60.0$ & $47.0$ \\
      & Distance to object       & $9.0$  & $1.0$  & $0.0$ & $0.0$  & $2.5$  \\
      & Action magnitude         & $39.0$ & $12.0$ & $7.0$ & $60.0$ & $29.5$ \\
      & Q-function variance      & $76.0$ & $37.0$ & $8.0$ & $70.0$ & $47.8$ \\
      & Q-value discrepancy~\citep{Huang2018EstablishingAT} & $81.0$ & $29.0$ &$11.0$ & $73.0$ & $48.5$ \\
      & \textbf{Ours}            & $\mathbf{83.0}$ & $\mathbf{43.0}$ & $13.0$ & $\mathbf{77.0}$ & $\mathbf{54.0}$ \\
    \bottomrule
  \end{tabular}
\end{table}

\subsection{Hardware}
\label{sec:exp-hw}

We further validated SCRIM on hardware, demonstrating that it can be feasibly learned through limited real-world interaction. %
The study investigated two tasks using Franka Panda manipulator for which demonstration data was collected via teleoperation: \emph{StackCube} (stack the red cube on the green cube; $70$
demonstrations) and \emph{TapeHang} (hang the tape on a hook; $100$ demonstrations).
The base policy for each task was obtained by finetuning $\pi_{0.5}$~\citep{intelligence2025pi_}, a vision-language-action model. The adversarial residual RL agent is trained with SAC~\cite{haarnoja18sac} directly on the robot, adding a bounded disturbance $\delta$ to the base-policy action (implementation details in Appendix~\ref{sec:appendix-hw-policy}).
Adversarial training was found to be effective with only $50$ rollouts for
StackCube and $150$ for TapeHang, indicating our method's effectiveness for real-world RL. %

\begin{wraptable}{r}{0.5\linewidth}
  \centering
  \small
  \setlength{\tabcolsep}{3pt}
  \renewcommand{\arraystretch}{0.9}
  \caption{\textbf{Hardware.} Success (\%) of the fine-tuned $\pi_{0.5}$ over
    $25$ rollouts. Residual-RL training episodes in parentheses.}
  \label{tab:results_hardware}
  \begin{tabular}{llcc}
    \toprule
    Data & Method & StackCube & TapeHang \\
    \midrule
    \multirow{3}{*}{100\%}
      & Uniform (BC)          & $56.0$ & $56.0$ \\
      & Residual RL~\cite{yuan2024policy} & $60.0$ ($50$) & $4.0$ ($150$) \\
      & \textbf{Ours}         & $\mathbf{76.0}$ & $\mathbf{88.0}$ \\
    \midrule
    \multirow{2}{*}{25\%}
      & Uniform sampling      & $52.0$ & $48.0$ \\
      & \textbf{Ours}         & $\mathbf{68.0}$ & $\mathbf{60.0}$ \\
    \bottomrule
  \end{tabular}
  \vspace{-8pt}
\end{wraptable}

\noindent \textbf{Weighted imitation and curation on hardware.}
We repeat the weighting and $25\%$ curation studies with the hardware SCRIM scores, re-finetuning $\pi_{0.5}$ on each resulting dataset and evaluating $25$ rollouts per task--policy pair (Table~\ref{tab:results_hardware}).
SCRIM-weighted training raises success from $56\%$ to $76\%$ on StackCube and from $56\%$ to $88\%$ on TapeHang. Residual RL~\citep{yuan2024policy} given the same number of real-world episodes as the adversary barely improves StackCube and severely degrades TapeHang, mirroring the matched-interaction results in simulation (Table~\ref{tab:residual_rl_matched}).
With a $25\%$ subset, SCRIM selection again outperforms uniform selection on both tasks, showing that SCRIM's benefits carry over from state-based simulated policies to a real-world VLA.

\noindent \textbf{Mean trajectory criticality flags fragile demonstrations.}
Because criticality is comparable across task trajectories, a demonstration's mean SCRIM grades its quality. A high mean value indicates a fragile success that a small
execution error would have turned into a failure, while a low mean marks a robust
one. Fig.~\ref{fig:hw_trajectory_quality} shows two such cube-stacking demonstrations: the high-mean one places the cube on the edge
of the support, where it barely balances and a slight deviation would topple it, whereas the low-mean one sets it down centered and stable.

\begin{figure}[H]
  \centering
\includegraphics[width=\textwidth]{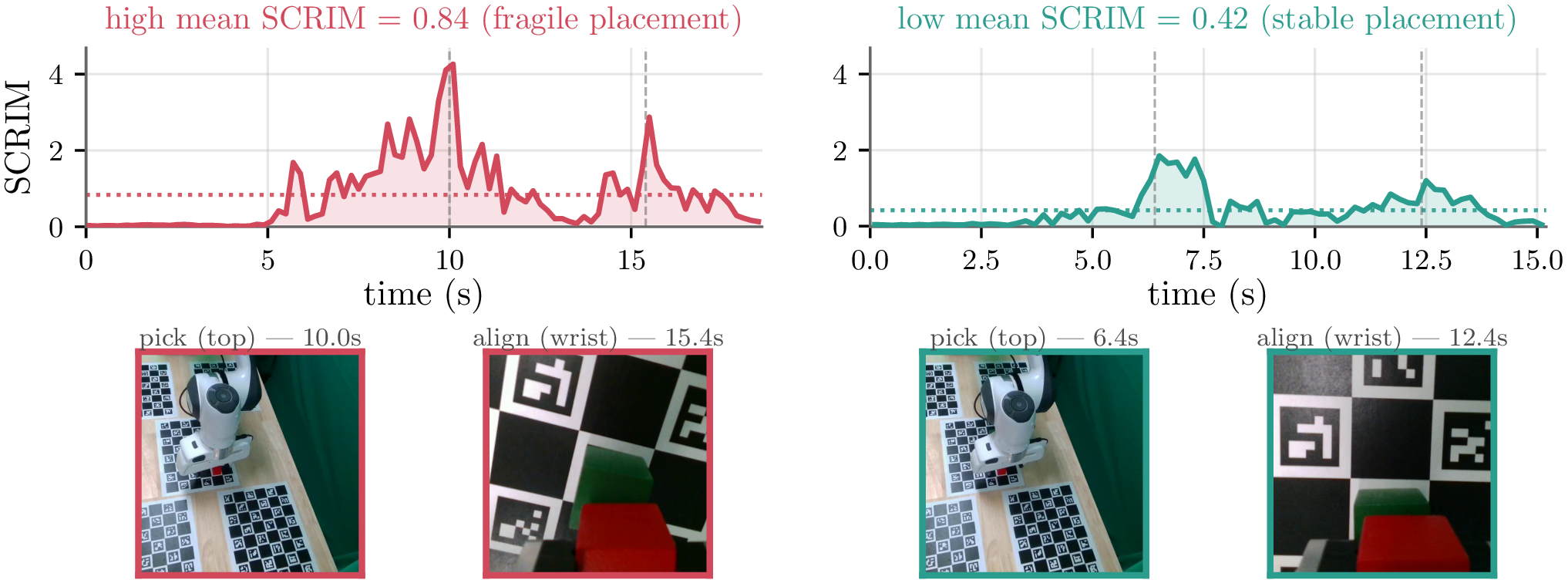}
  \caption{SCRIM distinguishes a brittle (left) from a robust (right) cube-stacking demonstration.}
\label{fig:hw_trajectory_quality}
\end{figure}

\section{Limitations and Future Work}
\label{sec:conclusion}
SCRIM estimates criticality without requiring expert supervision, but it still requires interaction with the environment. Although this is more targeted than training a full residual RL policy, it may still be difficult in real-world settings where failures are costly. Future work could address this by learning criticality in simulation, using real-to-sim reconstruction from demonstration, or
from logged failure data with offline RL.
SCRIM also improves imitation only within the support of the available demonstrations. It can identify and emphasize consequential states, but it cannot discover missing strategies or recovery behaviors absent from the dataset. An important direction is to use criticality to guide exploration or targeted data collection, focusing interactions near high-criticality states.
Finally, learning language-conditioned critics that generalize across objects, tasks, and embodiments is an exciting direction. Our faucet and hardware experiments suggest this may be feasible, since SCRIM can already be trained across task variations and with task instructions included in the observation.

\clearpage
\acknowledgments{We gratefully acknowledge research support from the NSF CAREER program (2240163) and the NASA ULI program. Ajay Suresha Sathya was supported by the European Union’s
Horizon Europe research and innovation programme
through a Marie Skłodowska-Curie Postdoctoral Fellowship
(101211945 — ExTRAORDiNary).}

\bibliography{manuscript}

\begin{thebibliography}{33}
\providecommand{\natexlab}[1]{#1}
\providecommand{\url}[1]{\texttt{#1}}
\expandafter\ifx\csname urlstyle\endcsname\relax
  \providecommand{\doi}[1]{doi: #1}\else
  \providecommand{\doi}{doi: \begingroup \urlstyle{rm}\Url}\fi

\bibitem[O'Neill et~al.(2024)O'Neill, Rehman, Maddukuri, Gupta, Padalkar, Lee,
  Pooley, Gupta, Mandlekar, Jain, et~al.]{ox2024open}
A.~O'Neill, A.~Rehman, A.~Maddukuri, A.~Gupta, A.~Padalkar, A.~Lee, A.~Pooley,
  A.~Gupta, A.~Mandlekar, A.~Jain, et~al.
\newblock Open {X-Embodiment}: Robotic learning datasets and {RT-X} models.
\newblock In \emph{2024 IEEE International Conference on Robotics and
  Automation (ICRA)}, pages 6892--6903. IEEE, 2024.

\bibitem[Khazatsky et~al.(2024)Khazatsky, Pertsch, Nair, Balakrishna, Dasari,
  Karamcheti, Nasiriany, Srirama, Chen, Ellis, Fagan, Hejna, Itkina, Lepert,
  Ma, Miller, Wu, Belkhale, Dass, Ha, Jain, Lee, Lee, Memmel, Park,
  Radosavovic, Wang, Zhan, Black, Chi, Hatch, Lin, Lu, Mercat, Rehman, Sanketi,
  Sharma, Simpson, Vuong, Walke, Wulfe, Xiao, Yang, Yavary, Zhao, Agia, Baijal,
  Castro, Chen, Chen, Chung, Drake, Foster, Gao, Herrera, Heo, Hsu, Hu,
  Jackson, Le, Li, Lin, Ma, Maddukuri, Mirchandani, Morton, Nguyen, O'Neill,
  Scalise, Seale, Son, Tian, Tran, Wang, Wu, Xie, Yang, Yin, Zhang, Bastani,
  Berseth, Bohg, Goldberg, Gupta, Gupta, Jayaraman, Lim, Malik,
  Martín-Martín, Ramamoorthy, Sadigh, Song, Wu, Yip, Zhu, Kollar, Levine, and
  Finn]{Khazatsky-RSS-24}
A.~Khazatsky, K.~Pertsch, S.~Nair, A.~Balakrishna, S.~Dasari, S.~Karamcheti,
  S.~Nasiriany, M.~K. Srirama, L.~Y. Chen, K.~Ellis, P.~D. Fagan, J.~Hejna,
  M.~Itkina, M.~Lepert, Y.~J. Ma, P.~T. Miller, J.~Wu, S.~Belkhale, S.~Dass,
  H.~Ha, A.~Jain, A.~Lee, Y.~Lee, M.~Memmel, S.~Park, I.~Radosavovic, K.~Wang,
  A.~Zhan, K.~Black, C.~Chi, K.~B. Hatch, S.~Lin, J.~Lu, J.~Mercat, A.~Rehman,
  P.~R. Sanketi, A.~Sharma, C.~Simpson, Q.~Vuong, H.~R. Walke, B.~Wulfe,
  T.~Xiao, J.~H. Yang, A.~Yavary, T.~Z. Zhao, C.~Agia, R.~Baijal, M.~G. Castro,
  D.~Chen, Q.~Chen, T.~Chung, J.~Drake, E.~P. Foster, J.~Gao, D.~A. Herrera,
  M.~Heo, K.~Hsu, J.~Hu, D.~Jackson, C.~Le, Y.~Li, R.~Lin, Z.~Ma, A.~Maddukuri,
  S.~Mirchandani, D.~Morton, T.~Nguyen, A.~O'Neill, R.~Scalise, D.~Seale,
  V.~Son, S.~Tian, E.~Tran, A.~E. Wang, Y.~Wu, A.~Xie, J.~Yang, P.~Yin,
  Y.~Zhang, O.~Bastani, G.~Berseth, J.~Bohg, K.~Goldberg, A.~Gupta, A.~Gupta,
  D.~Jayaraman, J.~J. Lim, J.~Malik, R.~Martín-Martín, S.~Ramamoorthy,
  D.~Sadigh, S.~Song, J.~Wu, M.~C. Yip, Y.~Zhu, T.~Kollar, S.~Levine, and
  C.~Finn.
\newblock {DROID: A Large-Scale In-The-Wild Robot Manipulation Dataset}.
\newblock In \emph{Proceedings of Robotics: Science and Systems}, Delft,
  Netherlands, July 2024.
\newblock \doi{10.15607/RSS.2024.XX.120}.

\bibitem[Chi et~al.(2023)Chi, Feng, Du, Xu, Cousineau, Burchfiel, and
  Song]{Chi-RSS-23}
C.~Chi, S.~Feng, Y.~Du, Z.~Xu, E.~Cousineau, B.~C. Burchfiel, and S.~Song.
\newblock {Diffusion Policy: Visuomotor Policy Learning via Action Diffusion}.
\newblock In \emph{Proceedings of Robotics: Science and Systems}, Daegu,
  Republic of Korea, July 2023.
\newblock \doi{10.15607/RSS.2023.XIX.026}.

\bibitem[Zitkovich et~al.(2023)Zitkovich, Yu, Xu, Xu, Xiao, Xia, Wu, Wohlhart,
  Welker, Wahid, Vuong, Vanhoucke, Tran, Soricut, Singh, Singh, Sermanet,
  Sanketi, Salazar, Ryoo, Reymann, Rao, Pertsch, Mordatch, Michalewski, Lu,
  Levine, Lee, Lee, Leal, Kuang, Kalashnikov, Julian, Joshi, Irpan, Ichter,
  Hsu, Herzog, Hausman, Gopalakrishnan, Fu, Florence, Finn, Dubey, Driess,
  Ding, Choromanski, Chen, Chebotar, Carbajal, Brown, Brohan, Arenas, and
  Han]{zitkovich23a}
B.~Zitkovich, T.~Yu, S.~Xu, P.~Xu, T.~Xiao, F.~Xia, J.~Wu, P.~Wohlhart,
  S.~Welker, A.~Wahid, Q.~Vuong, V.~Vanhoucke, H.~Tran, R.~Soricut, A.~Singh,
  J.~Singh, P.~Sermanet, P.~R. Sanketi, G.~Salazar, M.~S. Ryoo, K.~Reymann,
  K.~Rao, K.~Pertsch, I.~Mordatch, H.~Michalewski, Y.~Lu, S.~Levine, L.~Lee,
  T.-W.~E. Lee, I.~Leal, Y.~Kuang, D.~Kalashnikov, R.~Julian, N.~J. Joshi,
  A.~Irpan, B.~Ichter, J.~Hsu, A.~Herzog, K.~Hausman, K.~Gopalakrishnan, C.~Fu,
  P.~Florence, C.~Finn, K.~A. Dubey, D.~Driess, T.~Ding, K.~M. Choromanski,
  X.~Chen, Y.~Chebotar, J.~Carbajal, N.~Brown, A.~Brohan, M.~G. Arenas, and
  K.~Han.
\newblock Rt-2: Vision-language-action models transfer web knowledge to robotic
  control.
\newblock In J.~Tan, M.~Toussaint, and K.~Darvish, editors, \emph{Proceedings
  of The 7th Conference on Robot Learning}, volume 229 of \emph{Proceedings of
  Machine Learning Research}, pages 2165--2183. PMLR, 06--09 Nov 2023.
\newblock URL \url{https://proceedings.mlr.press/v229/zitkovich23a.html}.

\bibitem[Zhao et~al.(2023)Zhao, Kumar, Levine, and Finn]{Zhao-RSS-23}
T.~Z. Zhao, V.~Kumar, S.~Levine, and C.~Finn.
\newblock {Learning Fine-Grained Bimanual Manipulation with Low-Cost Hardware}.
\newblock In \emph{Proceedings of Robotics: Science and Systems}, Daegu,
  Republic of Korea, July 2023.
\newblock \doi{10.15607/RSS.2023.XIX.016}.

\bibitem[Black et~al.(2025)Black, Brown, Darpinian, Dhabalia, Driess, Esmail,
  Equi, Finn, Fusai, et~al.]{intelligence2025pi_}
K.~Black, N.~Brown, J.~Darpinian, K.~Dhabalia, D.~Driess, A.~Esmail, M.~Equi,
  C.~Finn, N.~Fusai, et~al.
\newblock $\pi_{0.5}$: a vision-language-action model with open-world
  generalization.
\newblock In \emph{Proceedings of The 9th Conference on Robot Learning}, volume
  305 of \emph{Proceedings of Machine Learning Research}, pages 17--40. PMLR,
  2025.

\bibitem[Brogliato(1996)]{brogliato1996nonsmooth}
B.~Brogliato.
\newblock \emph{Nonsmooth impact mechanics: models, dynamics and control}.
\newblock Springer, 1996.

\bibitem[Belkhale et~al.(2023)Belkhale, Cui, and Sadigh]{imitationdataquality}
S.~Belkhale, Y.~Cui, and D.~Sadigh.
\newblock Data quality in imitation learning.
\newblock In A.~Oh, T.~Naumann, A.~Globerson, K.~Saenko, M.~Hardt, and
  S.~Levine, editors, \emph{Advances in Neural Information Processing Systems},
  volume~36, pages 80375--80395. Curran Associates, Inc., 2023.
\newblock URL
  \url{https://proceedings.neurips.cc/paper_files/paper/2023/file/fe692980c5d9732cf153ce27947653a7-Paper-Conference.pdf}.

\bibitem[Johannink et~al.(2019)Johannink, Bahl, Nair, Luo, Kumar, Loskyll,
  Ojea, Solowjow, and Levine]{johannink2019residual}
T.~Johannink, S.~Bahl, A.~Nair, J.~Luo, A.~Kumar, M.~Loskyll, J.~A. Ojea,
  E.~Solowjow, and S.~Levine.
\newblock Residual reinforcement learning for robot control.
\newblock In \emph{2019 international conference on robotics and automation
  (ICRA)}, pages 6023--6029. IEEE, 2019.

\bibitem[Ankile et~al.(2025)Ankile, Jiang, Duan, Shi, Abbeel, and
  Nagabandi]{ankile2025residual}
L.~Ankile, Z.~Jiang, R.~Duan, G.~Shi, P.~Abbeel, and A.~Nagabandi.
\newblock Residual off-policy {RL} for finetuning behavior cloning policies.
\newblock \emph{arXiv preprint arXiv:2509.19301}, 2025.

\bibitem[Yuan et~al.(2025)Yuan, Mu, Tao, Fang, Zhang, and Su]{yuan2024policy}
X.~Yuan, T.~Mu, S.~Tao, Y.~Fang, M.~Zhang, and H.~Su.
\newblock Policy decorator: Model-agnostic online refinement for large policy
  model.
\newblock In \emph{International Conference on Learning Representations
  (ICLR)}, 2025.

\bibitem[Ding et~al.(2025)Ding, Veer, Leung, Cao, and
  Pavone]{Ding2025SurprisePA}
W.~Ding, S.~Veer, K.~Leung, Y.~Cao, and M.~Pavone.
\newblock Surprise potential as a measure of interactivity in driving
  scenarios.
\newblock \emph{arXiv preprint arXiv:2502.05677}, 2025.

\bibitem[Huang et~al.(2018)Huang, Bhatia, Abbeel, and
  Dragan]{Huang2018EstablishingAT}
S.~H. Huang, K.~Bhatia, P.~Abbeel, and A.~D. Dragan.
\newblock Establishing appropriate trust via critical states.
\newblock In \emph{2018 IEEE/RSJ International Conference on Intelligent Robots
  and Systems (IROS)}, pages 3929--3936, 2018.
\newblock \doi{10.1109/IROS.2018.8593649}.

\bibitem[Hejna et~al.(2025)Hejna, Bhateja, Jiang, Pertsch, and
  Sadigh]{remixdata}
J.~Hejna, C.~A. Bhateja, Y.~Jiang, K.~Pertsch, and D.~Sadigh.
\newblock Remix: Optimizing data mixtures for large scale imitation learning.
\newblock In P.~Agrawal, O.~Kroemer, and W.~Burgard, editors, \emph{Proceedings
  of The 8th Conference on Robot Learning}, volume 270 of \emph{Proceedings of
  Machine Learning Research}, pages 145--164. PMLR, 06--09 Nov 2025.
\newblock URL \url{https://proceedings.mlr.press/v270/hejna25a.html}.

\bibitem[Parekh et~al.(2026)Parekh, Nemlekar, and Losey]{Parekh2025TowardsBB}
S.~Parekh, H.~Nemlekar, and D.~P. Losey.
\newblock Towards balanced behavior cloning from imbalanced datasets.
\newblock \emph{Autonomous Robots}, 50\penalty0 (1), 2026.
\newblock \doi{10.1007/s10514-025-10237-0}.

\bibitem[Agia et~al.(2025)Agia, Sinha, Yang, Antonova, Pavone, Nishimura,
  Itkina, and Bohg]{Agia2025CUPIDCD}
C.~Agia, R.~Sinha, J.~Yang, R.~Antonova, M.~Pavone, H.~Nishimura, M.~Itkina,
  and J.~Bohg.
\newblock {CUPID}: Curating data your robot loves with influence functions.
\newblock In \emph{Proceedings of The 9th Conference on Robot Learning}, volume
  305 of \emph{Proceedings of Machine Learning Research}, pages 2907--2932.
  PMLR, 2025.

\bibitem[Dass et~al.(2026)Dass, Khaddaj, Engstrom, Madry, Ilyas, and
  Mart{\'i}n-Mart{\'i}n]{Dass2025DataMILSD}
S.~Dass, A.~Khaddaj, L.~Engstrom, A.~Madry, A.~Ilyas, and
  R.~Mart{\'i}n-Mart{\'i}n.
\newblock {DataMIL}: Selecting data for robot imitation learning with
  datamodels.
\newblock In \emph{The Fourteenth International Conference on Learning
  Representations (ICLR)}, 2026.

\bibitem[Kuhar et~al.(2023)Kuhar, Cheng, Chopra, Bronars, and
  Xu]{pmlr-v229-kuhar23a}
S.~Kuhar, S.~Cheng, S.~Chopra, M.~Bronars, and D.~Xu.
\newblock Learning to discern: Imitating heterogeneous human demonstrations
  with preference and representation learning.
\newblock In J.~Tan, M.~Toussaint, and K.~Darvish, editors, \emph{Proceedings
  of The 7th Conference on Robot Learning}, volume 229 of \emph{Proceedings of
  Machine Learning Research}, pages 1437--1449. PMLR, 06--09 Nov 2023.
\newblock URL \url{https://proceedings.mlr.press/v229/kuhar23a.html}.

\bibitem[Hejna et~al.(2025)Hejna, Mirchandani, Balakrishna, Xie, Wahid,
  Tompson, Sanketi, Shah, Devin, and Sadigh]{Hejna2025RobotDC}
J.~Hejna, S.~Mirchandani, A.~Balakrishna, A.~Xie, A.~Wahid, J.~Tompson, P.~R.
  Sanketi, D.~Shah, C.~Devin, and D.~Sadigh.
\newblock Robot data curation with mutual information estimators.
\newblock In \emph{Proceedings of Robotics: Science and Systems}, 2025.
\newblock \doi{10.15607/RSS.2025.XXI.023}.

\bibitem[Chen et~al.(2025)Chen, Lessing, Liu, and Finn]{Chen2025CuratingDU}
A.~S. Chen, A.~M. Lessing, Y.~Liu, and C.~Finn.
\newblock Curating demonstrations using online experience.
\newblock In \emph{Proceedings of Robotics: Science and Systems}, 2025.
\newblock \doi{10.15607/RSS.2025.XXI.071}.

\bibitem[Ciftci et~al.(2025)Ciftci, Chiu, Feng, Sukhatme, and
  Bansal]{Ciftci2024SAFEGILSG}
Y.~U. Ciftci, D.~Chiu, Z.~Feng, G.~S. Sukhatme, and S.~Bansal.
\newblock {SAFE-GIL}: {SAFE}ty guided imitation learning for robotic systems.
\newblock In \emph{2025 IEEE International Conference on Robotics and
  Automation (ICRA)}, pages 3559--3566, 2025.
\newblock \doi{10.1109/ICRA55743.2025.11128298}.

\bibitem[Qiu et~al.(2026)Qiu, Ciftci, and Bansal]{Qiu2025SafetyAwareIL}
L.~Qiu, Y.~U. Ciftci, and S.~Bansal.
\newblock Safety-aware imitation learning via {MPC}-guided disturbance
  injection.
\newblock \emph{IEEE Robotics and Automation Letters}, 11\penalty0
  (6):\penalty0 7222--7229, 2026.
\newblock \doi{10.1109/LRA.2026.3685942}.

\bibitem[Brandfonbrener et~al.(2023)Brandfonbrener, Tu, Singh, Welker, Boodoo,
  Matni, and Varley]{Brandfonbrener2022VisualBT}
D.~Brandfonbrener, S.~Tu, A.~Singh, S.~Welker, C.~Boodoo, N.~Matni, and
  J.~Varley.
\newblock Visual backtracking teleoperation: A data collection protocol for
  offline image-based reinforcement learning.
\newblock In \emph{2023 IEEE International Conference on Robotics and
  Automation (ICRA)}, pages 11336--11342, 2023.
\newblock \doi{10.1109/ICRA48891.2023.10161096}.

\bibitem[Sindhwani et~al.(2018)Sindhwani, Tu, and
  Khansari]{Sindhwani2018LearningCV}
V.~Sindhwani, S.~Tu, and M.~Khansari.
\newblock Learning contracting vector fields for stable imitation learning.
\newblock \emph{ArXiv}, abs/1804.04878, 2018.
\newblock URL \url{https://api.semanticscholar.org/CorpusID:4861590}.

\bibitem[Mehta et~al.(2025)Mehta, Ciftci, Ramachandran, Bansal, and
  Losey]{Mehta2024StableBCCC}
S.~A. Mehta, Y.~U. Ciftci, B.~Ramachandran, S.~Bansal, and D.~P. Losey.
\newblock Stable-{BC}: Controlling covariate shift with stable behavior
  cloning.
\newblock \emph{IEEE Robotics and Automation Letters}, 10\penalty0
  (2):\penalty0 1952--1959, 2025.
\newblock \doi{10.1109/LRA.2025.3526439}.

\bibitem[Xiao et~al.(2026)Xiao, Lin, Peng, Xue, He, Luo, Xie, Hu, Fan, Shi, and
  Zhu]{vla_res_rl}
W.~Xiao, H.~Lin, A.~Peng, H.~Xue, T.~He, Z.~Luo, Y.~Xie, F.~Hu, L.~Fan, G.~Shi,
  and Y.~Zhu.
\newblock Self-improving vision-language-action models with data generation via
  residual {RL}.
\newblock In \emph{The Fourteenth International Conference on Learning
  Representations (ICLR)}, 2026.
\newblock URL \url{https://openreview.net/forum?id=eUGoqrZ6Ea}.

\bibitem[Ankile et~al.(2025)Ankile, Simeonov, Shenfeld, Torne, and
  Agrawal]{Ankile2024FromIT}
L.~Ankile, A.~Simeonov, I.~Shenfeld, M.~Torne, and P.~Agrawal.
\newblock From imitation to refinement -- residual {RL} for precise assembly.
\newblock In \emph{2025 IEEE International Conference on Robotics and
  Automation (ICRA)}, 2025.
\newblock \doi{10.1109/ICRA55743.2025.11127442}.

\bibitem[Yu et~al.(2020)Yu, Quillen, He, Julian, Hausman, Finn, and
  Levine]{yu2020meta}
T.~Yu, D.~Quillen, Z.~He, R.~Julian, K.~Hausman, C.~Finn, and S.~Levine.
\newblock Meta-world: A benchmark and evaluation for multi-task and meta
  reinforcement learning.
\newblock In \emph{Conference on robot learning}, pages 1094--1100. PMLR, 2020.

\bibitem[Gu et~al.(2023)Gu, Xiang, Li, Ling, Liu, Mu, Tang, Tao, Wei, Yao,
  et~al.]{gumaniskill2}
J.~Gu, F.~Xiang, X.~Li, Z.~Ling, X.~Liu, T.~Mu, Y.~Tang, S.~Tao, X.~Wei,
  Y.~Yao, et~al.
\newblock {ManiSkill2}: A unified benchmark for generalizable manipulation
  skills.
\newblock In \emph{The Eleventh International Conference on Learning
  Representations (ICLR)}, 2023.

\bibitem[Shafiullah et~al.(2022)Shafiullah, Cui, Altanzaya, and
  Pinto]{shafiullah2022behavior}
N.~M. Shafiullah, Z.~Cui, A.~A. Altanzaya, and L.~Pinto.
\newblock Behavior transformers: Cloning $k$ modes with one stone.
\newblock In \emph{Advances in Neural Information Processing Systems},
  volume~35, pages 22955--22968, 2022.

\bibitem[Karino et~al.(2020)Karino, Ohmura, and
  Kuniyoshi]{Karino2020Identifying}
I.~Karino, Y.~Ohmura, and Y.~Kuniyoshi.
\newblock Identifying critical states by the action-based variance of expected
  return.
\newblock In \emph{Artificial Neural Networks and Machine Learning -- ICANN
  2020}, volume 12396 of \emph{Lecture Notes in Computer Science}, pages
  366--378. Springer, 2020.
\newblock \doi{10.1007/978-3-030-61609-0_29}.

\bibitem[Haarnoja et~al.(2018)Haarnoja, Zhou, Abbeel, and
  Levine]{haarnoja18sac}
T.~Haarnoja, A.~Zhou, P.~Abbeel, and S.~Levine.
\newblock Soft actor-critic: Off-policy maximum entropy deep reinforcement
  learning with a stochastic actor.
\newblock In J.~Dy and A.~Krause, editors, \emph{Proceedings of the 35th
  International Conference on Machine Learning}, volume~80 of \emph{Proceedings
  of Machine Learning Research}, pages 1861--1870. PMLR, 10--15 Jul 2018.
\newblock URL \url{https://proceedings.mlr.press/v80/haarnoja18b.html}.

\bibitem[Kakade and Langford(2002)]{kakade2002approximately}
S.~Kakade and J.~Langford.
\newblock Approximately optimal approximate reinforcement learning.
\newblock In \emph{Proceedings of the Nineteenth International Conference on
  Machine Learning (ICML)}, pages 267--274, 2002.

\end{thebibliography}

\clearpage
\appendix

\section{Appendix}
\label{sec:appendix}

\subsection{SCRIM algorithm}\label{sec:appendix-algorithm}
Algorithm~\ref{alg:scrim} summarizes SCRIM end to end. Phase~1 trains an
adversarial disturbance policy $\mu_\theta$ and a value critic in the residual
MDP $\widetilde{\mathcal{M}}^{\,\epsilon}_{\pi}$ with soft actor-critic
(SAC)~\cite{haarnoja18sac}; Phase~2 uses the converged conservative critic
$\bar{Q}$ and the deterministic adversary $\delta_\pi$ to score every
demonstration state with the criticality estimator $\Delta Q(s)$
(Eq.~\ref{eq:delta-q}). The diffusion/visual base policy, action chunking, and
exploration-scheduling details of any particular instantiation are abstracted
away; the base controller is simply the frozen imitation policy $\pi$. The
returned scores feed the sampling-distribution calibration of
Appendix~\ref{sec:appendix-calibration}.

\begin{algorithm}[H]
\caption{SCRIM (State CRIticality Measure)}
\label{alg:scrim}
\begin{algorithmic}[1]
\Require base imitation policy $\pi$, per-coordinate budget $\epsilon$, demonstrations $\mathcal{D}=\{(s,\pi_E(s))\}$, discount $\gamma$, Polyak rate $\tau$
\Ensure criticality scores $\{\Delta Q(s)\}_{s\in\mathcal{D}}$
\Statex \emph{// Phase 1: train adversary and critic in $\widetilde{\mathcal{M}}^{\,\epsilon}_{\pi}$ with SAC}
\State initialize twin critics $Q_{\phi_1},Q_{\phi_2}$, targets $\bar{\phi_i}\!\leftarrow\!\phi_i$, adversary $\mu_\theta$, temperature $\alpha$, replay buffer $\mathcal{B}$
\Statex \hspace{\algorithmicindent} with disturbance policy $\pi_{\mathrm{adv}}(\cdot\mid s)$, mean $\delta_\theta(s)=\epsilon\odot\tanh(\mu_\theta(s))\in\Delta_\epsilon$
\Repeat
  \State sample $\delta\sim\mu_\theta(\cdot\mid s)$; execute $a=\operatorname{clip}_{\mathcal{A}}\big(\pi(s)+\delta\big)$; observe next state $s'$
  \State reward $r\leftarrow -1$ if the rollout succeeds, else $0$; store $(s,\delta,r,s',d)$ in $\mathcal{B}$
  \State sample a minibatch from $\mathcal{B}$; draw $\delta'\sim\mu_\theta(\cdot\mid s')$
  \State $y \leftarrow r + \gamma\,(1-d)\big[\,\min_i Q_{\bar{\phi_i}}\big(s',\pi(s')+\delta'\big) - \alpha\log\pi_{\mathrm{adv}}(\delta'\mid s')\,\big]$
  \State update critics: minimize $\textstyle\sum_i\big(Q_{\phi_i}\big(s,\pi(s)+\delta\big) - y\big)^2$
  \State update adversary $\mu_\theta$: maximize $\min_i Q_{\phi_i}\big(s,\pi(s)+\delta_\theta(s)\big) - \alpha\log\pi_{\mathrm{adv}}(\delta_\theta(s)\mid s)$
  \State update temperature $\alpha$ toward the target entropy
  \State Polyak update targets: $\bar{\phi_i}\leftarrow\tau\phi_i + (1-\tau)\bar{\phi_i}$
\Until{converged}
\State set conservative critic $\bar{Q}\leftarrow\min_i Q_{\bar{\phi_i}}$ and deterministic adversary $\delta_\pi(s)\leftarrow\epsilon\odot\tanh(\mu_\theta(s))$
\Statex \emph{// Phase 2: score demonstration states around the expert action}
\For{each $(s,\pi_E(s))\in\mathcal{D}$}
  \State $\Delta Q(s)\leftarrow \bar{Q}\big(s,\pi_E(s)+\delta_\pi(s)\big) - \bar{Q}\big(s,\pi_E(s)\big)$ \Comment{Eq.~\ref{eq:delta-q}}
\EndFor
\State \Return $\{\Delta Q(s)\}$
\end{algorithmic}
\end{algorithm}

\subsection{Theoretical motivation for SCRIM}\label{sec:appendix-theory}
This section motivates three design choices of SCRIM: criticality is measured
on demonstration states, around the expert action, and with the critic of the
learned policy $\pi$.

\vspace{0.1cm} \noindent \textbf{Setup.}
Let $\pi_E$ and $\pi$ be deterministic. So that larger values are better, in
this section we use the success reward $r^{+} \coloneqq -r$ and write $J(\pi)$
and $Q^{\pi}(s, a)$ for the return and action-value (Eq.~\ref{eq:q-function})
of $\pi$ under $r^{+}$. Let
$d_E(s) = (1-\gamma) \sum\nolimits_{t \ge 0} \gamma^t \Pr(s_t = s \mid \pi_E)$ be
the discounted state distribution of the expert. We define the criticality of
a state as
\begin{equation}
  C_\epsilon(s)
    = \max_{\|u\| \le \epsilon}
      \Big[ Q^{\pi}\big(s, \pi_E(s)\big) - Q^{\pi}\big(s, \pi_E(s) + u\big) \Big] ,
  \label{eq:app-criticality}
\end{equation}
the worst-case drop in the learner's value that an in-budget error around the
expert action can cause. Here $\|\cdot\|$ is the max norm, so that the
constraint is the budget set $\Delta_\epsilon$ of
Eq.~\ref{eq:residual-action} with a shared budget $\epsilon$.

\begin{theorem}[Criticality-weighted performance bound]\label{thm:criticality-bound}
Assume that $Q^{\pi}(s,\cdot)$ is smooth around $\pi_E(s)$, with a
second-order Taylor remainder $O(\|a - \pi_E(s)\|^3)$ uniformly in $s$; that
$\pi_E(s)$ is an interior local maximizer of $Q^{\pi}(s,\cdot)$; and that
$\|\pi(s) - \pi_E(s)\| \le \eta$ for all $s$. Then
\begin{equation}
  J(\pi_E) - J(\pi)
    \;\le\; \frac{1}{(1-\gamma)\,\epsilon^2}\,
      \mathbb{E}_{s \sim d_E}\Big[
      C_\epsilon(s)\, \big\| \pi(s) - \pi_E(s) \big\|^2 \Big]
      + O\big( \epsilon\eta^2 + \eta^3 \big) .
  \label{eq:app-gap}
\end{equation}
\end{theorem}

\begin{proof}
For each state $s$, let $a^* = \pi_E(s)$, $a_\pi = \pi(s)$, and
$\delta_a(s) = a_\pi - a^*$. By the performance difference
lemma~\citep{kakade2002approximately},
\begin{equation*}
  J(\pi_E) - J(\pi)
    = \frac{1}{1-\gamma}\, \mathbb{E}_{s \sim d_E}\big[
        Q^{\pi}(s, a^*) - Q^{\pi}(s, a_\pi) \big] .
\end{equation*}
A Taylor expansion of $Q^{\pi}(s,\cdot)$ around $a^*$ gives
\begin{equation*}
  Q^{\pi}(s, a_\pi)
    = Q^{\pi}(s, a^*)
      + \nabla_a Q^{\pi}(s, a^*)^\top \delta_a(s)
      + \tfrac{1}{2}\, \delta_a(s)^\top \nabla_a^2 Q^{\pi}(s, a^*)\, \delta_a(s)
      + O\big( \|\delta_a(s)\|^3 \big) .
\end{equation*}
The first-order term vanishes because $a^*$ is a local maximizer, and the
remainder is $O(\eta^3)$ because $\|\delta_a(s)\| \le \eta$. With
$q_s(v) \coloneqq \tfrac{1}{2}\, v^\top \big( -\nabla_a^2 Q^{\pi}(s, a^*) \big) v$,
\begin{align*}
  Q^{\pi}(s, a^*) - Q^{\pi}(s, a_\pi)
    &= q_s\big(\delta_a(s)\big) + O\big( \eta^3 \big) \\
    &\le \max_{\|v\| \le \|\delta_a(s)\|} q_s(v) + O\big( \eta^3 \big) \\
    &= \frac{\|\delta_a(s)\|^2}{\epsilon^2} \max_{\|u\| \le \epsilon} q_s(u)
       + O\big( \eta^3 \big) \\
    &\le C_\epsilon(s)\, \frac{\|\delta_a(s)\|^2}{\epsilon^2}
       + O\big( \epsilon\eta^2 + \eta^3 \big) ,
\end{align*}
where the third line rescales $v = (\|\delta_a(s)\| / \epsilon)\, u$. The last
line uses $\max\nolimits_{\|u\| \le \epsilon} q_s(u) \le C_\epsilon(s) + O(\epsilon^3)$,
which follows from the same Taylor expansion, since
$Q^{\pi}(s, a^*) - Q^{\pi}(s, a^* + u) = q_s(u) + O(\epsilon^3)$ for
$\|u\| \le \epsilon$. Taking the expectation over $s \sim d_E$ gives
Eq.~\ref{eq:app-gap}.
\end{proof}

\vspace{0.1cm} \noindent \textbf{Implications for SCRIM.}
The performance difference lemma writes the gap as the learner's value loss at
the expert action, averaged over expert states. SCRIM follows this
decomposition: it scores demonstration states, perturbs around the expert
action $\pi_E(s)$, and uses the critic of the learned policy. This is why the
critic and adversary are trained with $\pi$ as the base controller but
evaluated around $\pi_E(s)$ (Eq.~\ref{eq:delta-q}). Taking the worst case over
the budget makes the bound hold whatever the direction of the imitation error.
Eq.~\ref{eq:app-gap} then weights the imitation error at each state by its
criticality, so errors at critical states cost the most. This motivates
weighting demonstrations by criticality.

\vspace{0.1cm} \noindent \textbf{Scope.}
The theorem concerns the idealized score of Eq.~\ref{eq:app-criticality}, which
uses the exact $Q^{\pi}$ and an exact maximization. SCRIM approximates it:
$\bar{Q}$ is learned with SAC against an adversarial continuation rather than
the unperturbed continuation of $\pi$, and the maximization is done by an adversary
trained around $\pi(s)$ rather than $\pi_E(s)$. The theorem therefore
motivates the design of SCRIM; it does not guarantee the performance of the
trained estimator.

\subsection{Simulation setup and per-environment details}~\label{sec:appendix-sim}
All actions are normalized to $[-1,1]^{d_a}$.
\vspace{0.1cm} \noindent \textbf{MetaWorld (Hammer).}
Observations are $84{\times}84$ RGB images and a $4$-dimensional proprioceptive vector, and the
action is a $4$-dimensional end-effector command. The base policy is a deterministic CNN$+$MLP
trained on $20$ expert demonstrations. The SAC adversary is trained for $5{\times}10^5$ steps with a disturbance bound of $0.20$ per axis.
\vspace{0.1cm} \noindent \textbf{ManiSkill2 (Stack Cube, Turn Faucet, PegInsertion).}
These tasks use state observations and a Behavior Transformer~\cite{shafiullah2022behavior} base
policy with $20$ time-steps of history. StackCube, TurnFaucet, and PegInsertion are trained on
$974$, $1008$, and $986$ demonstrations respectively (Turn Faucet spanning $10$ faucet models).
Table~\ref{tab:app_budget} lists the disturbance bound and the interaction budgets of each task.

\begin{table}[ht]
  \caption{\textbf{Disturbance bounds and interaction budgets.} Environment interactions used
    to train the SCRIM adversary behind Table~1
    baselines.}
  \label{tab:app_budget}
  \centering
  \begin{tabular}{lcccc}
    \toprule
     & StackCube & TurnFaucet & PegInsertion & Hammer \\
    \midrule
    Disturbance bound $\epsilon$ (per axis) & 0.005 & 0.005 & 0.002 & 0.20 \\
    SCRIM adversary & 200k & 100k & 1M & 500k \\
    \bottomrule
  \end{tabular}
\end{table}

\subsection{Sampling-distribution calibration}~\label{sec:appendix-calibration}
This expands the calibration pipeline summarized in Section~\ref{sec:method}.
Denote by $g_i$ the criticality score at transition $i$. Each score is $z$-score
standardized,
$z_i = (g_i - \operatorname{mean}(g)) / (\operatorname{std}(g) + \xi)$ with a small
$\xi > 0$ for numerical stability, and mapped to a distribution at temperature
$T$,
\begin{equation}
  p_i(T) = \frac{\exp(z_i / T)}{\sum_{j=1}^{N} \exp(z_j / T)} .
  \label{eq:softmax}
\end{equation}
Because the softmax is shift and scale invariant,
standardization only fixes a common temperature scale across measures whose raw
scores have very different magnitudes.

We set $T$ through the sharpness of $p$, measured by its fractional effective
sample size,
\begin{equation}
  \mathrm{ESS}(T) = \frac{\big(\sum_i p_i\big)^2}{\sum_i p_i^2}
    = \frac{1}{\sum_i p_i^2},
  \qquad
  \mathrm{ESS}\%(T) = \frac{100}{N}\,\mathrm{ESS}(T) \in (0, 100] ,
  \label{eq:ess}
\end{equation}
which equals $100\%$ for the uniform distribution and decreases as mass
concentrates on a few transitions. Since $\mathrm{ESS}\%(T)$ increases
monotonically in $T$, we select the temperature $T^\star$ that attains a target
$\mathrm{ESS}\% = \tau$ by bisection and use $p \equiv p(T^\star)$. The target $\tau$ is a hyperparameter.

Table~\ref{tab:app_ess} lists the selected target and the resulting temperature for SCRIM on each task, and the success rate of SCRIM-weighted BC across targets. Performance is robust to the choice of target across the tested targets, above uniform BC throughout.

\begin{table}[ht]
  \caption{\textbf{Sampling target.} Top: selected ESS\% target $\tau$ and resulting temperature
    $T^\star$ for SCRIM. Bottom: success rate (\%) of SCRIM-weighted BC across ESS\% targets.}
  \label{tab:app_ess}
  \centering
  \begin{tabular}{lcccc}
    \toprule
     & StackCube & TurnFaucet & PegInsertion & Hammer \\
    \midrule
    ESS\% target $\tau$ & 44.7 & 40.6 & 69.7 & 50.0 \\
    Temperature $T^\star$ & 1.97 & 1.72 & 3.75 & 5.62 \\
    \bottomrule
  \end{tabular}

  \vspace{6pt}
  \begin{tabular}{lcccc}
    \toprule
    \multicolumn{5}{l}{\textit{StackCube} (uniform BC 68.0)} \\
    ESS\% & 8.7 & 25.2 & 44.7 & 72.0 \\
    Success (\%) & 87.0 & 84.0 & 91.0 & 87.0 \\
    \midrule
    \multicolumn{5}{l}{\textit{TurnFaucet} (uniform BC 39.0)} \\
    ESS\% & 10.0 & 19.8 & 40.6 & 75.0 \\
    Success (\%) & 46.0 & 43.0 & 47.0 & 43.0 \\
    \bottomrule
  \end{tabular}
\end{table}

\subsection{Baseline details}~\label{sec:appendix-baselines}
\vspace{0.1cm} \noindent \textbf{Uniform sampling (i.i.d. setting)}
Uniform sampling treats every state as equally important, sampling uniformly from the dataset during training and selecting a subset of the data with equal probability for each state. This is the standard approach in imitation learning and serves as a natural baseline for our method.

\vspace{0.1cm} \noindent \textbf{Distance to object of interest}
A purely geometric heuristic that uses the distance between the end-effector and
the object of interest as a proxy for criticality: states close to the object are deemed critical and states far from it are not. This baseline relies only on the scene geometry and ignores both the expert action and the value function. Also, it can only be applied when we have the information about the target object's position.

\vspace{0.1cm} \noindent \textbf{Action magnitude}
We observe that the critical states that can be derailed by a small error are often those requiring intricate manipulation (e.g.\ grasping or insertion), which involve smaller commanded actions than the coarse movements such as reaching that exhibit larger action magnitudes. This baseline therefore scores criticality by the (negative) magnitude of the expert action, $-\| a \|$, treating low-magnitude actions as more critical.

\vspace{0.1cm} \noindent \textbf{Policy Q-function variance}
We use the spread of the base BC policy's Q function
$Q^{\pi_b}$ across actions as a measure of criticality, following
\citet{Karino2020Identifying}. Concretely, we sample $M$ actions uniformly
from the action space and score a state by the empirical variance
$\operatorname{Var}_{a \sim \mathcal{U}(\mathcal{A})}\!\big[ Q^{\pi_b}(s, a) \big]$.
High variance indicates that $Q^{\pi_b}$ varies strongly with the action
choice and the state is sensitive to action errors, whereas low variance
indicates the value is similar across many actions and the state is
relatively forgiving.

\vspace{0.1cm} \noindent \textbf{Q-value discrepancy}
\citet{Huang2018EstablishingAT} propose the gap between the
best achievable Q-value at a state and the expected value over all actions, $\max_{a} Q^{\pi_b}(s, a) - \frac{1}{|\mathcal{A}|} \sum_{a} Q^{\pi_b}(s, a)$, as a measure of criticality.
A large discrepancy means that a random action yields a substantially worse long-term return than the optimal action, marking the state as highly critical, whereas a small discrepancy means that most actions yield comparable value.

For the last two baselines, $Q^{\pi_b}$ is trained independently of SCRIM by fitted Q evaluation. The frozen base policy is rolled out in the environment with the sparse success reward, episodes ending on success, and $\gamma = 0.97$.
A twin-Q critic is regressed onto the targets $r + \gamma \min_i Q_{\bar{\phi}_i}(s', \pi_b(s'))$
with Polyak-averaged target networks and no actor.

\begin{figure}[t]
  \centering
  \resizebox{\linewidth}{!}{\input{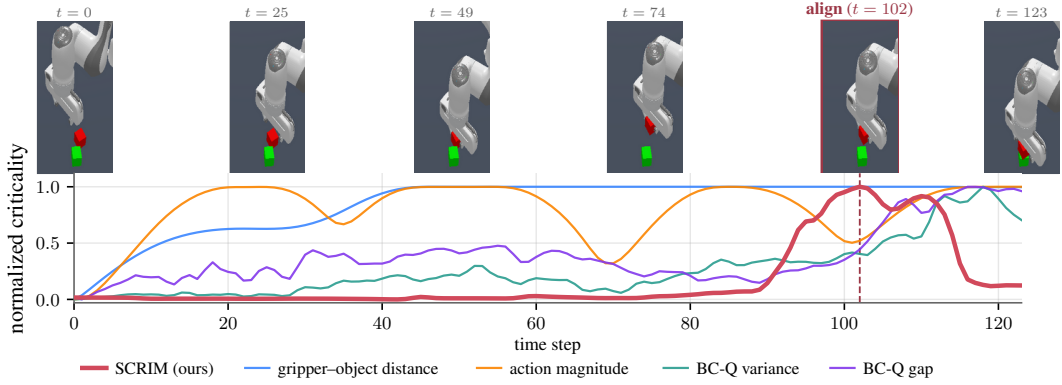}}
  \caption{\textbf{Per-method criticality along StackCube (ManiSkill) task demonstration.}All measures min--max normalized to a common scale.}
  \label{fig:stackcube_sim_criticality_plots}
\end{figure}

\begin{figure}[t]
  \centering
  \resizebox{\linewidth}{!}{\input{figures/app_criticality_comparison_turnfaucet_sim/criticality_comparison.pgf}}
  \caption{\textbf{Per-method criticality along TurnFaucet (ManiSkill) task demonstration.}All measures min--max normalized to a common scale.}
  \label{fig:turnfaucet_sim_criticality_plots}
\end{figure}

\begin{figure}[t]
  \centering
  \resizebox{\linewidth}{!}{\input{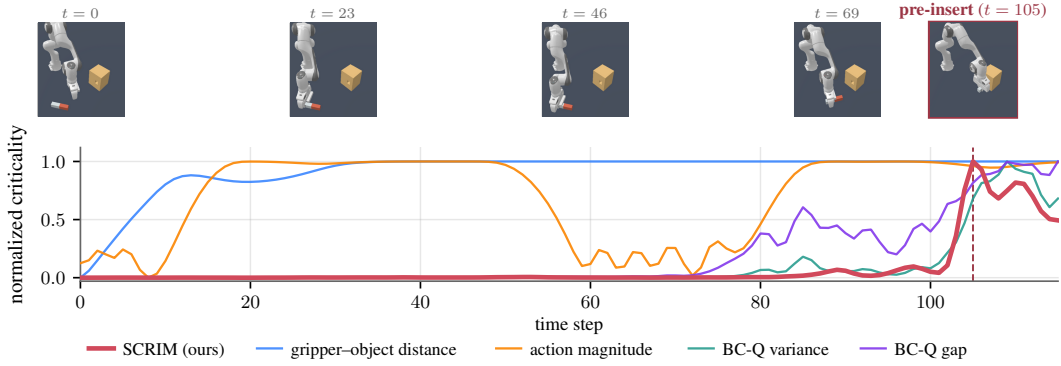}}
  \caption{\textbf{Per-method criticality along PegInsertionSide (ManiSkill) task demonstration.}All measures min--max normalized to a common scale.}
  \label{fig:peginsertion_criticality_plots}
\end{figure}

\begin{figure}[t]
  \centering
  \resizebox{\linewidth}{!}{\input{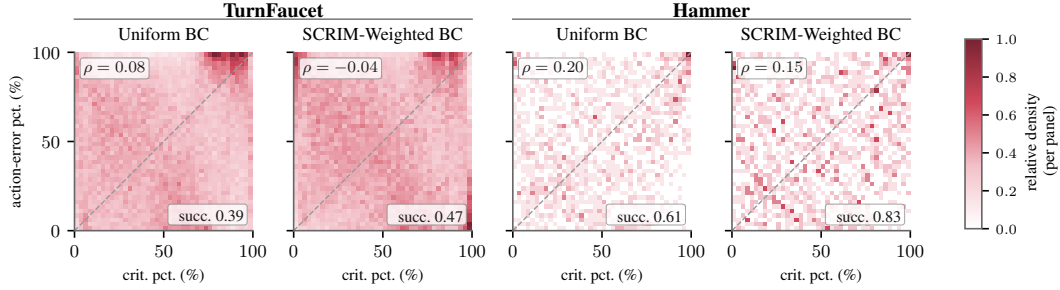}}
  \caption{\textbf{Action error vs.\ state criticality.} Histograms of BC action-error percentile against SCRIM- percentile on two tasks, each under Uniform BC and SCRIM-weighted BC.}
  \label{fig:error_vs_criticality_app}
\end{figure}

\begin{figure}[t]
  \centering
  \resizebox{\linewidth}{!}{\input{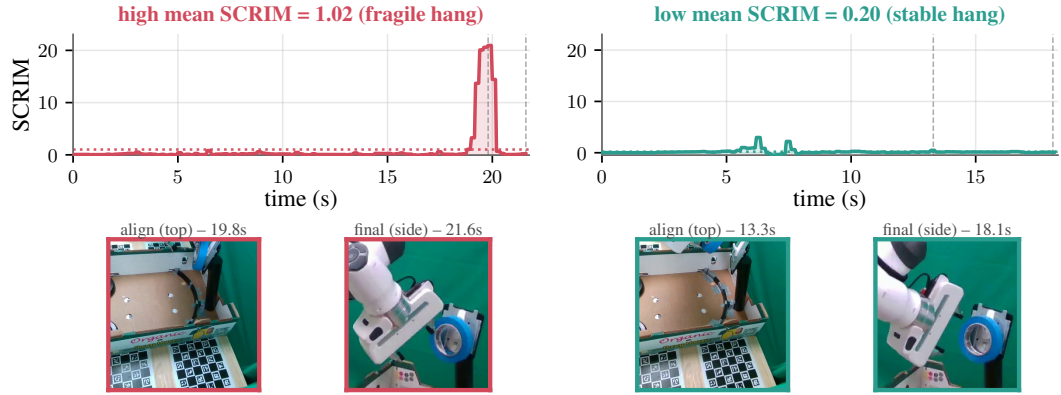}}
  \caption{SCRIM distinguishes a brittle (left) from a robust (right) tape-hang demonstration.}
  \label{fig:hw_trajectory_quality_tapehang}
\end{figure}

\subsection{Hardware setup and adversary details}~\label{sec:appendix-hw-policy}
\vspace{0.1cm} \noindent \textbf{Base policy.}
Each task's $\pi_{0.5}$ base policy was finetuned for $30$k training steps with an action-chunk
horizon of $16$, and is served at $15$\,Hz with replanning every $4$ steps. The policy observes
three $256{\times}256$ camera views (top, wrist, and side), an $8$-dimensional proprioceptive
vector, and the task language prompt, and acts in a $7$-dimensional delta end-effector pose with
gripper.

\vspace{0.1cm} \noindent \textbf{Adversary.}
Each policy query returns an action chunk together with the mean-pooled VLM prefix embedding, a
$2048$-dimensional summary of the images, proprioception, and prompt. This embedding is sent as the
observation to the adversarial residual RL agent, trained with SAC~\cite{haarnoja18sac} directly on
the robot. The adversary acts on the $4$-step executed chunk, so its action has dimension
$28=7\times4$. Its tanh-squashed output is scaled by a task-specific residual envelope $\rho$ and
added to the base-policy chunk before clipping to the robot safety limits. The translation residual
scale is $0.007$\,m for both tasks; the gripper scale is $0.4$; the rotation scales are
$0.015$\,rad on all axes for StackCube and $(0.0275,0.0131,0.0216)$\,rad for TapeHang. The critic
networks use layer normalization, and the replay buffer is warm-started from the demonstration
dataset.

\vspace{0.1cm} \noindent \textbf{Real-world SCRIM computation.}
For each demonstration chunk we replay the frozen $\pi_{0.5}$ policy to obtain the base chunk
$a_{\mathrm{base}}$ and encode the demonstrated expert chunk as
$\delta_{\mathrm{demo}}=\operatorname{clip}((a_{\mathrm{expert}}-a_{\mathrm{base}})/\rho,-1,1)$.
Let $\delta^\star$ denote the deterministic adversary output at the same state embedding. The
hardware criticality score used in Fig.~\ref{fig:teaser} and Fig.~\ref{fig:hw_trajectory_quality} is
\begin{equation}
  \Delta Q_{\mathrm{demo}}
  =
  \bar Q\big(s, a_{\mathrm{base}}+\rho\odot(\delta_{\mathrm{demo}}+\delta^\star)\big)
  -
  \bar Q\big(s, a_{\mathrm{base}}+\rho\odot\delta_{\mathrm{demo}}\big).
  \label{eq:hw-delta-q}
\end{equation}
This is the same failure-value increase as Eq.~\ref{eq:delta-q}, evaluated around the demonstrated
action in the chunked real-robot action space. For plotting, each chunk-level value is held constant
over its corresponding control ticks.

\subsection{Learning SCRIM without physical interaction via real2sim}~\label{sec:appendix-real2sim}
The interaction SCRIM requires need not take place on the physical system. To test this, we
reconstructed our hardware StackCube demonstrations in simulation via real2sim, trained a
vision-based diffusion policy on the reconstructed data, and trained the SCRIM adversary and
critic entirely in simulation against this policy. Despite the real2sim gap, the
simulation-derived SCRIM scores correlate with the scores obtained from hardware interaction and
recover the same critical phases, peaking around grasping and placement. Using these scores for
SCRIM-weighted training improves the policy's success rate in simulation from $40\%$ to $66\%$.
This suggests that SCRIM's interaction can be shifted to a suitable surrogate environment, and
that approximate criticality estimates obtained under model mismatch remain useful for
downstream policy improvement. Whether this simulated policy improvement transfers back to the
physical robot (sim2real) remains to be verified, and we leave this investigation to future work.

\subsection{Choosing the perturbation bound}
\label{app:eps}
The perturbation bound is shared across action coordinates ($\epsilon_j = \epsilon$ in Eq.~3). We
choose it by rolling out the frozen base policy with uniform random residuals $a = \pi(s) + \epsilon u$, $u \sim \mathcal{U}[-1, 1]^{d}$, resampled at every step (50 episodes per $\epsilon$; Table~\ref{tab:app_random_residual}). If random residuals leave success unchanged, $\epsilon$ is too small for any disturbance to matter; if they already collapse success, every state appears critical and $\Delta Q$ stops discriminating between states. Table~\ref{tab:app_eps} reports SCRIM-weighted BC across $\epsilon$ at a fixed ESS\% target.

\begin{table}[ht]
  \caption{\textbf{Random-residual calibration (StackCube).} Success rate (\%) of the frozen base
    policy under uniform random residuals of size $\epsilon$.}
  \label{tab:app_random_residual}
  \centering
  \begin{tabular}{lccccccc}
    \toprule
    $\epsilon$ & 0 & 0.001 & 0.003 & 0.005 & 0.01 & 0.02 & 0.05 \\
    \midrule
    Success (\%) & 72.0 & 64.0 & 42.0 & 24.0 & 14.0 & 0.0 & 0.0 \\
    \bottomrule
  \end{tabular}
\end{table}

As shown in Table~\ref{tab:app_eps}, SCRIM is nevertheless robust over a broad range of bounds.

\begin{table}[H]
  \caption{\textbf{Sensitivity to the perturbation bound.} Success rate (\%) of SCRIM-weighted BC
    for each $\epsilon$ at a fixed ESS\% target.}
  \label{tab:app_eps}
  \centering
  \begin{tabular}{lccccc}
    \toprule
    $\epsilon$ & 0.003 & 0.005 & 0.01 & 0.02 & 0.05\\
    \midrule
    \multicolumn{6}{l}{\textit{StackCube} (uniform BC 68.0)} \\
    Success (\%) & 83.0 & 91.0 & 73.0 & 77.0 & 73.0 \\
    \midrule
    \multicolumn{6}{l}{\textit{TurnFaucet} (uniform BC 39.0)} \\
    Success (\%) & 42.0 & 47.0 & 37.0 & 42.0 & 42.0 \\
    \bottomrule
  \end{tabular}
\end{table}

\subsection{Causal validation of SCRIM}
\label{app:causal}
We test whether SCRIM predicts where action errors matter by perturbing the base policy at only
a fraction $k$ of steps. At every step, SCRIM is computed around the base action,
$\bar{Q}(s, \pi(s) + \delta_\pi(s)) - \bar{Q}(s, \pi(s))$, and the disturbance is applied when the
score exceeds its $(1-k)$ percentile over 25 clean rollouts. As a control, the same fraction of
steps is selected at random. The disturbance is either the trained $\epsilon$-bounded adversary
$\delta_\pi(s)$ or uniform $\epsilon$-noise $\mathcal{U}[-\epsilon, \epsilon]^{d}$
(100 episodes per condition; Table~\ref{tab:app_causal}).

\begin{table}[H]
  \caption{\textbf{Score-gated disturbance.} Drop in success rate (percentage points) relative to
    clean rollouts when the disturbance is applied at a fraction $k$ of steps.}
  \label{tab:app_causal}
  \centering
  \begin{tabular}{lcccc}
    \toprule
     & StackCube & TurnFaucet & PegInsertion & Hammer \\
    \midrule
    Gated fraction $k$ & 5\% & 5\% & 30\% & 10\% \\
    Clean success (\%) & 69.0 & 29.0 & 13.0 & 66.0 \\
    \midrule
    Adversary at high-SCRIM states & 13.0 & 13.0 & 13.0 & 20.0 \\
    Adversary at random states & 6.0 & 4.0 & 6.0 & 1.0 \\
    Uniform $\epsilon$-noise at high-SCRIM states & 1.0 & 0.0 & 0.0 & 3.0 \\
    \bottomrule
  \end{tabular}
\end{table}

Applied at high-SCRIM states, the adversary degrades success at least twice as much as at the same fraction of random states, while the uniform $\epsilon$-noise at the same states barely changes success. SCRIM therefore captures vulnerability to specific failure-inducing directions, not generic sensitivity to random noise.

\subsection{Robustness to critic error}
\label{app:critic}
We corrupt every critic query inside the SCRIM score of Appendix~\ref{app:causal} with independent
Gaussian noise $\mathcal{N}(0, \sigma^2)$, $\sigma = c \cdot \mathrm{std}(Q)$, where
$\mathrm{std}(Q)$ is the standard deviation of $\bar{Q}(s, \pi(s))$ over clean rollouts. Gating
thresholds are recomputed under the noise, so the gated fraction stays at $k$. The score SNR is
$\mathrm{std}(\Delta Q) / (\sqrt{2}\,\sigma)$. We also gate and attack with the adversary and
critic checkpoint saved halfway through training (Table~\ref{tab:app_critic}).

\begin{table}[ht]
  \caption{\textbf{Degraded critics.} Drop in success rate (percentage points) from the adversary
    at high-SCRIM states, with the score SNR in parentheses, at the gated fractions of
    Table~\ref{tab:app_causal}.}
  \label{tab:app_critic}
  \centering
  \begin{tabular}{lcccc}
    \toprule
     & StackCube & TurnFaucet & PegInsertion & Hammer \\
    \midrule
    Noise scale $c = 0$ & 13.0 & 13.0 & 13.0 & 20.0 \\
    $c = 0.003$ & 16.0 (14.3) & 12.0 (4.1) & 13.0 (12.2) & 20.0 (18.4) \\
    $c = 0.01$ & 14.0 (4.3) & 9.0 (1.2) & 8.0 (3.7) & 19.0 (5.5) \\
    $c = 1$ & 3.0 (0.04) & $-1.0$ (0.01) & 5.0 (0.04) & 2.0 (0.06) \\
    \midrule
    Halfway checkpoint & 4.0 & 1.0 & 5.0 & 23.0 \\
    Random states (reference) & 6.0 & 4.0 & 6.0 & 1.0 \\
    \bottomrule
  \end{tabular}
\end{table}

At $c = 0.003$ (SNR 4.1--18.4), the drops remain comparable to the noise-free values, and at
$c = 0.01$ they remain 8--19 points. Once the noise approaches the local $\Delta Q$ scale
(SNR $\le 1$), the effect becomes task-dependent, and at $c = 1$ the drops collapse to $-1$ to 5
points. SCRIM therefore tolerates moderate zero-mean query noise only while the local value gap
remains resolvable. Undertraining has a task-dependent effect: on Hammer, the halfway checkpoint
is as effective as the final one, whereas on the ManiSkill tasks its drops
fall to at most 5 points, no more than random gating. We therefore do not claim general
robustness to critic undertraining.

\end{document}